%% file: paper.tex
\documentclass[12pt,twoside]{article}
\input{setting.tex}

\DeclareMathSizes{9}{8}{7}{5}

\title{\textbf{Deep Linear Networks can Benignly Overfit when Shallow Ones 
Do}}

\author{
Niladri S. Chatterji \\ 
Stanford University \\
niladri@cs.stanford.edu \\
      \and
Philip M. Long  \\
Google \\
plong@google.com 
}

\date{\today}

\begin{document}
\maketitle
\begin{abstract}
We bound the excess risk of interpolating deep linear networks trained using gradient flow. 
In a setting previously used to establish risk bounds for the minimum $\ell_2$-norm interpolant, we show that randomly initialized deep linear networks can closely approximate or even match known bounds for the minimum $\ell_2$-norm interpolant. 
Our analysis also reveals that interpolating deep linear models have exactly the same 
conditional variance
as the minimum $\ell_2$-norm solution.  
Since the noise affects the excess risk only
through the conditional variance, this implies that 
depth does not
improve the algorithm's ability to ``hide the noise''.
Our simulations verify that aspects of our bounds reflect typical
behavior for simple data distributions.  We also find that similar phenomena
are seen in simulations with ReLU networks, although the situation there is more nuanced.
\end{abstract}

\section{Introduction} 
Recent empirical studies~\citep{zhang2016understanding,belkin2019reconciling} have brought to light the surprising phenomenon that overparameterized 
neural network models trained with variants of gradient descent generalize well despite perfectly fitting
noisy data.
This seemingly violates the 
once widely accepted
principle that learning algorithms
should trade off between some measure of the regularity of a model, and its fit to the data. To understand this, a rich line of research has emerged to establish conditions under which extreme overfitting---fitting the data perfectly---is benign
in simple models~\citep[see][]{belkin2018overfitting,hastie2019surprises,bartlett2020benign}. Another closely connected thread of research to understand generalization 
leverages
the
recognition that 
training by gradient descent 
engenders
an implicit bias~\citep[see][]{neyshabur2014search,soudry2018implicit,ji2018risk}. These results can be
paraphrased as follows: training until the loss is driven to zero will produce a model
that, among models that interpolate the data, 
minimizes some data-independent regularity criterion.  

Our paper continues this study of benign overfitting but with a more complex model class, deep linear networks. Deep linear networks are often studied theoretically \citep[see, e.g.,][]{DBLP:journals/corr/SaxeMG13,DBLP:conf/icml/AroraCH18}, because
some of the relevant characteristics of deep learning
in the presence of nonlinearities are also present in linear networks but in a setting
that is more amenable to analysis.  The analyses of linear networks have included
a number of results on implicit bias \citep[see, e.g.,][]{azulay2021implicit,MinEtAl21}.  Recently, one of these analyses~\citep{azulay2021implicit}, of two-layer
networks trained by gradient flow with a ``balanced'' initialization, was leveraged in an
analysis of benign overfitting \citep{chatterji2022interplay}.   (For a mapping $x \rightarrow  x W v$
parameterized by a hidden layer $W \in \R^{d\times m}$ and an output layer $v \in \R^{m\times 1}$,
initial values of $v$ and $W$ are balanced if $v v^{\top} = W^\top W$.)
\citet{MinEtAl21} analyzed implicit bias in
two-layer linear networks
under more general conditions including the
unbalanced case.

In this paper, we analyze benign overfitting in deep linear networks of arbitrary depth trained by gradient
flow.  Our first main result is a bound on the excess risk.  The bound is in terms
of some characteristics of the joint distribution of the training data previously used to analyze
linear regression with the
standard parameterization, including notions of the effective rank of the covariance matrix, and it holds under
similar conditions on the data distribution.  Another key quantity used in the bound
concerns the linear map $\Theta$ computed by the network after training---it is the
norm of the projection of this map onto the subspace orthogonal to the span of the
training examples.  This norm can further be bounded in terms of its value at initialization,
and a quantity that reflects how rapidly training converged.  
%LMC In the case that the function computed by the network at initialization is
% identically zero, we show that its generalization after training until interpolation
% is the same as that of the minimum $\ell_2$-norm interpolator with the standard parameterization.
In contrast with previous analyses on two-layer networks~\citep{chatterji2022interplay}, this analysis holds whether this initialization is balanced or not.

Our second main result is a high-probably risk bound that holds for networks in which
the first and last layers are initialized randomly, and the middle layers are all
initialized to the identity.  Our bound holds whenever the scale of the initialization
of the first layer is small enough, and the scale of the initialization of the last
layer is large enough.  This includes the extreme case where the first layer is initialized
to zero.  As the scale of the initialization of the first layer goes to zero, our
bound approaches the known bound for the minimum $\ell_2$-norm interpolator with the standard 
parameterization.
Our final main theoretical result illustrates our bounds using a simple covariance matrix
used in previous work \citep{bartlett2020benign,chatterji2022foolish} which might be viewed as a canonical case where overfitting
is benign for linear regression with the standard parameterization.

These bounds were obtained in the absence of a precise characterization
of the implicit bias of gradient flow for deep linear networks, or
a closed-form formula for the model produced.

A key point of our analysis is that the projection of the linear map $\Theta$ computed by the interpolating network onto
the span of the rows of the design matrix $X$ is exactly equal to minimum $\ell_2$-norm interpolant $\Theta_{\ell_2}$.
The risk of $\Theta$ naturally decomposes into contributions from this projection and
$\Theta_{X^{\perp}} = \Theta - \Theta_{\ell_2}$.  We can use previous analyses of $\Theta_{\ell_2}$
to bound the former.

\begin{figure}[H]
     \centering
     \begin{subfigure}[b]{\textwidth}
         \centering
         \includegraphics[width=\textwidth]{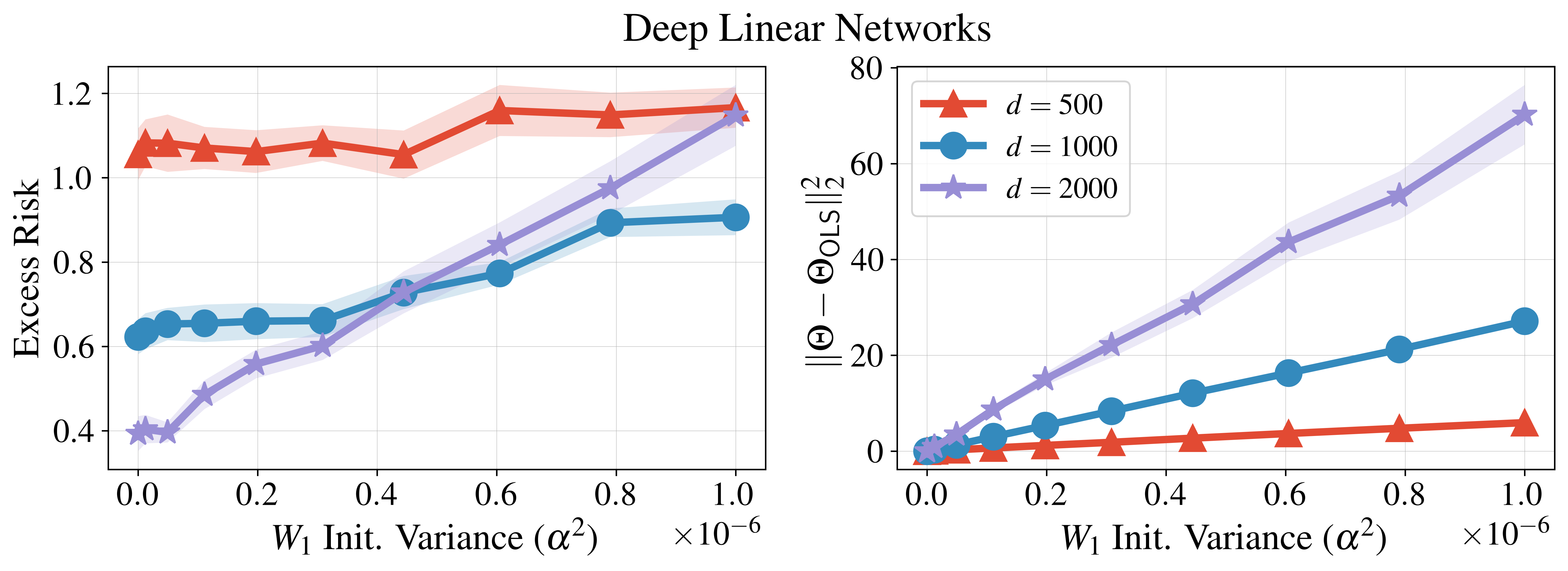}
         \label{fig:linear_models_risk_versus_alpha}
     \end{subfigure}
     \begin{subfigure}[b]{\textwidth}
         \centering
         \includegraphics[width=\textwidth]{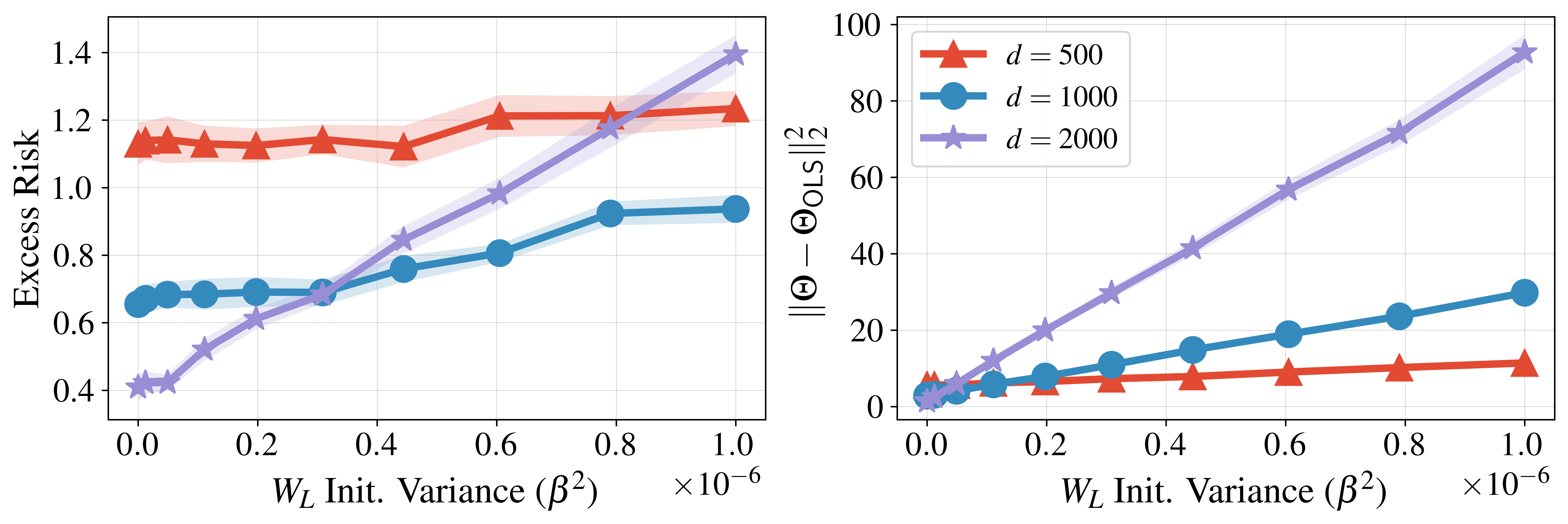}
     \end{subfigure}
        \caption{Three-layer linear networks trained by gradient descent on data generated by an underlying linear model. The model is trained on $n=100$ points drawn from the generative model $y = x\Theta^\star + \omega$, where $x \sim \mathsf{N}(0,\Sigma)$ and $\omega \sim \mathsf{N}(0,1)$. The excess risk is defined as $\E_{x}\left[\lv x\Theta-x\Theta^\star\rv^2\right]$. We empirically find that  when the initialization variance of either the first layer $(\alpha^2)$ or the last layer $(\beta^2)$ is close to zero, the final solution is close to the minimum $\ell_2$-norm interpolator and suffers small excess risk. While when the initialization variance is large, that is, when the network is initialized away from the origin, the excess risk is larger due to the component of the final solution outside the span of the data.
        Additional details in Section~\ref{s:simulations}.}
        \label{fig:linear_models_risk_versus_init_scale}
\end{figure}
Figure~\ref{fig:linear_models_risk_versus_init_scale} contains
plots from simulation experiments where the excess risk of
a deep linear model increases with the scale of the initialization
of the first layer, as in the upper bounds of our
analysis. A similar effect is also seen when the first layer is initialized
at a unit scale, and the scale of the initialization of the last layer varies.
In both cases, we also see that as the function computed by the
network at initialization approaches the zero function, the
trained model approaches the minimum $\ell_2$-norm interpolant.

Figure~\ref{fig:relu_models_risk_versus_alpha} includes plots of
analogous experiments with networks with ReLU nonlinearities.
As in the linear case the excess risk increases with the scale
of the initialization of the first layer, but we do not see
a significant increase in excess risk with the scale of the
initialization of the last layer.
\begin{figure}[H]
    \centering
    \includegraphics[width=\textwidth]{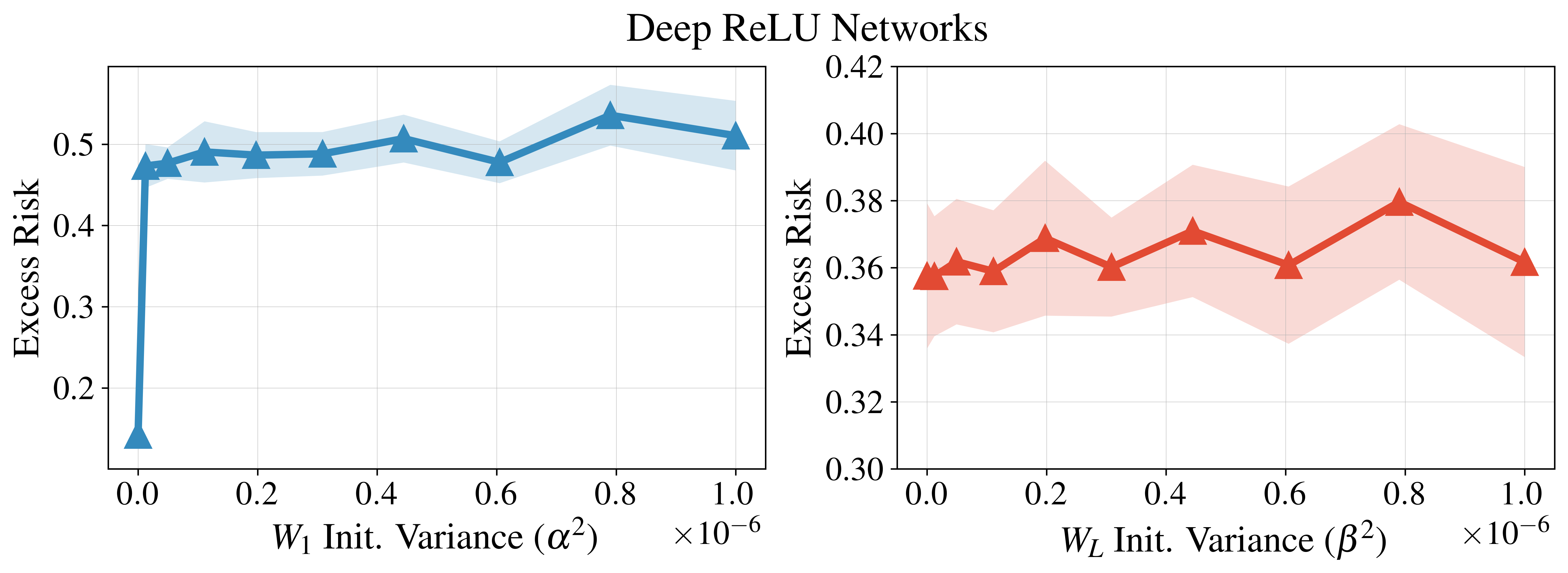}
    \caption{Three-layer ReLU networks trained by gradient descent on data generated by an underlying two-layer ReLU teacher network. The model is trained on $n=500$ points drawn from the generative model $y = f^\star(x) + \omega$, where $f^\star$ is a two-layer ReLU network with width $50$, $x \sim \mathsf{N}(0,I_{10 \times 10})$ and $\omega \sim \mathsf{N}(0,1)$. The excess risk is defined as $\E_{x}\left[\lv f(x)-f^\star(x)\rv^2\right]$. In ReLU models we find that the risk scales differently as we scale the initialization variance of the first layer $(\alpha^2)$ and that of the last layer $(\beta^2)$. When we scale $\alpha^2$, similar to deep linear models we find that risk is smaller for smaller values of $\alpha^2$. However, this is not the case when we scale $\beta^2$. This highlights a surprising asymmetry in the role played by the initialization scales of the different layers in ReLU networks. For additional details about the experiment see Section~\ref{s:simulations}.}
    \label{fig:relu_models_risk_versus_alpha}
\end{figure}
More details of the experiments are described in
Section~\ref{s:simulations}.

Intuitively, the harm from overfitting arises from fitting the noise, and the effect
of fitting the noise is analyzed in the conditional variance of the estimator.  In the setting studied
here, as in linear regression with the standard parameterization, the conditional variance is entirely determined by the projection of $\Theta$ onto the span of the rows of the data matrix $X$ which is equal to $\Theta_{\ell_2}$. Thus, when learning deep linear
networks with quadratic loss, aspects of training that
affect the inductive bias, such as the initialization, architecture, etc., do not
affect this variance term---no matter how they are chosen, the distribution of
the variance term is determined by $\Theta_{\ell_2}$.
To see an effect of implicit bias in deep linear networks on
the consequence of fitting the noise, we must analyze a loss
function other than the quadratic loss.

Our upper bounds reveal no benefit in representing linear transformations by deep networks, and,
in our simulations, we see no benefit with random initialization. 
This is because non-zero random initialization usually contributes additional error to the bias as the random initialization is typically a poor guess for the regression function.  (In rare cases it
could reduce the bias, though, if by chance it approximates the regression function.)

Our analysis also leverages the effect of imbalanced initialization on implicit bias---our treatment partially extends the results by \citet{MinEtAl21}
from the two-layer case to the deep case, and then combines them with our 
general risk bound.

\paragraph{Organization.} 
In Section~\ref{s:prelim} we describe our problem setting and our assumptions. Then in Section~\ref{s:main_results} we present our main results and in Sections~\ref{s:proof_main_theorem}, 
\ref{s:proof_random_init_theorem} and \ref{s:spike_proof}
we prove these results. We provide additional simulations and simulation details in Section~\ref{s:simulations}. We conclude with a discussion in Section~\ref{s:discussion}. In Appendix~\ref{s:additional_related_work} we highlight other related work on benign overfitting, implicit bias, and on linear networks. Finally, in Appendix~\ref{a:optimization_ab_initialization} we present omitted technical details.

\section{Preliminaries}\label{s:prelim}
This section includes notational conventions and a description of the setting.
\subsection{Notation}
\label{s:notation}
Given a vector $v$, let $\lv v \rv$ denote its Euclidean norm. Given a matrix $M$, let $\lv M \rv$ denote its Frobenius norm and let $\lv M \rv_{op}$ denote its operator norm.
  For any $j \in \N$, we denote the set $\{ 1,\ldots,j \}$ by $[j]$. We will use $c,  c', c_1, c_x, \ldots$ to denote positive 
absolute
constants, which may take different values in
different contexts.  

\subsection{Setting}
\label{s:setting}
We analyze linear regression with $d$ inputs and
$q$ outputs from $n$ examples.    Throughout the paper we assume that $d > n$. Although we assume throughout that the input dimension $d$ is finite, it is straightforward to extend our results to infinite $d$.

Let $X \in \mathbb{R}^{n \times d}$ be the data matrix, and $Y \in \mathbb{R}^{n\times q}$ be the response matrix,
%LMC . That is, 
and
let $x_1,\ldots,x_n \in \mathbb{R}^{1\times d}$
be the rows of $X$ and $y_1,\ldots,y_n \in \mathbb{R}^{1\times q}$ be the rows of $Y$. 

For random $(x,y)\in \mathbb{R}^{1\times d}\times\mathbb{R}^{1\times q}$, let
\[\Theta^\star \in \argmin_{\Theta \in\R^{d\times q}} \E_{(x,y)} \left[\lv y-x\Theta\rv^2\right]\] 
be an arbitrary optimal linear regressor. We let $\Omega = Y - X\Theta^\star \in \mathbb{R}^{n\times q}$ be the noise matrix. 

Define the \emph{excess risk} of an estimate $\Theta \in \R^{d\times q}$ to be
\begin{align*}
    \mathsf{Risk}(\Theta) := \E_{x,y}\left[\lv y-x\Theta\rv^2-\lv y-x \Theta^\star\rv^2\right],
\end{align*}
where $x,y$ are test samples that are independent 
of $\Theta$.

Denote the second moment matrix of the 
covariates
by $\Sigma :=\mathbb{E}[x^{\top}x] \in \mathbb{R}^{d\times d}$ with eigenvalues $\lambda_{1}\ge \ldots \ge \lambda_{d}\ge 0$.
We will use the following definitions of the
``effective rank'' that~\citet{bartlett2020benign} previously used in the analysis of the excess risk of the minimum $\ell_2$-norm interpolant.
\begin{definition}
Given any $j \in [d]$, define $s_j := \sum_{i>j}\lambda_i$ and 
\begin{align*}
    r_j := \frac{s_j}{\lambda_{j+1}} \qquad \text{and} \qquad R_j := \frac{s_j^2}{\sum_{i>j}\lambda_i^2}.
\end{align*}
\end{definition}

We define the index $k$ below. The value of $k$ shall help determine what we consider the ``tail'' of the covariance matrix. 
\begin{definition}\label{def:k_star}
For a large enough constant $b$ (that will be fixed henceforth), define
\begin{align*}
    k := \min\{j\ge 0: r_j\ge b n\},
\end{align*}
where the minimum of the empty set is defined as $\infty$.
\end{definition}

We are now ready to introduce the assumptions of our paper.

\paragraph{Assumptions.} Let $c_x$ and $c_y$ denote absolute constants.
\begin{enumerate}[label={(A.\arabic*)},
                  leftmargin=0.65in]
\item \label{assumption:first}The samples $(x_1,y_1),\ldots,(x_n,y_n)$
are drawn i.i.d.
    \item The covariates $x$ and responses $y$ are mean-zero.
    \item The covariates 
    $x$ satisfy
    $x= \Sigma^{1/2}u$, where $u$ is isotropic and has components that are independent $c_x$-sub-Gaussian random variables, that is, for all $\phi \in \R^d$
    \begin{align*}
        \E\left[\exp\left(\phi^{\top}u\right)\right]\le \exp\left(c_x \lv \phi\rv^2/2\right).
    \end{align*}
    \item\label{assumption:noise_subgaussian} The difference $y-x\Theta^\star$ is $c_y$-sub-Gaussian, conditionally on $x$;
    that is, for all $\phi \in \R^{q}$
    \begin{align*}
        \E_{y}\left[\exp\left(\phi^{\top}(y-x\Theta^\star)\right)\; \big| \;x\right]\le \exp\left(c_y \lv\phi\rv^2/2\right)
    \end{align*}
    (note that this implies that $\E\left[y \mid x\right] = x\Theta^\star$ and
    $\E\left[ \lv y - x\Theta^\star \rv^2 \right] \leq c q$). 
    \item \label{assumption:last} Almost surely, the projection of the data $X$ on the space orthogonal to any eigenvector of $\Sigma$ spans a space of dimension $n$.
\end{enumerate}
All the constants going forward may depend on the values of $c_x$ and  $c_y$. The assumptions made here are standard in the benign overfitting literature~\citep[see][]{bartlett2020benign,chatterji2022interplay}. They are satisfied for example in the case where $x$ is a mean-zero Gaussian whose covariance $\Sigma$ has full rank,
$d > n$,
and the noise $y-x\Theta^\star$ is independent and Gaussian.

\subsection{Deep Linear Models}
\label{s:training}

We analyze linear models represented by deep
linear networks with $m$ hidden units at each layer.
We denote the weight matrices by $W_1,\ldots, W_L$, where $W_1 \in \R^{m \times d}$, $W_2,\ldots, W_{L-1} \in \R^{m\times m}$, and $W_{L} \in \R^{q \times m}$. 
The standard representation of the network's linear transformation, denoted by
$\Theta \in \R^{d \times q}$, is
\begin{align*}
    \Theta & = \left(W_{L} \cdots W_{1}\right)^{\top} \in \R^{d\times q}.
\end{align*}
Define $P_{X}$ to be the projection onto the
row span of $X$, that is, $P_{X} := X^{\top}(XX^{\top})^{-1}X$. 
Let
\begin{align*}
 \Theta_X := P_{X} \Theta \quad \text{and} \quad    \Theta_{X^{\perp}} := (I - P_X) \Theta.
\end{align*}

For $n$ datapoints $(x_1,y_1),\ldots,(x_n,y_n)$, where $x_i \in \R^{1\times d}$ and $y_i \in \R^{1\times q}$, the training loss is given by
\begin{align*}
    \cL(\Theta) & := \sum_{i=1}^n \lv y_i - x_i\Theta\rv^2 = \lv Y - X\Theta \rv^2.
\end{align*}

We will analyze the generalization properties of deep linear
models trained with gradient flow, that is, for all $j \in [L]$,
\begin{align*}
    \frac{\mathrm{d}W_{j}^{(t)}}{\mathrm{d} t} & = - \nabla_{W_{j}^{(t)}} \cL(\Theta^{(t)}).
\end{align*}
We study the following random initialization scheme in our paper.
\begin{definition}(Random initialization)\label{def:random_init}
Given $\alpha, \beta >0$, the entries of the first layer $W_1^{(0)}$ and the last layer $W_L^{(0)}$ are initialized using i.i.d.\ draws from
$\mathsf{N}(0,\alpha^2)$ and $\mathsf{N}(0,\beta^2)$ respectively.  The remaining layers
$W_2^{(0)},\ldots,W_{L-1}^{(0)}$ are initialized to the identity $I_m$.
\end{definition}
A similar initialization scheme has been studied previously~\citep{DBLP:conf/iclr/ZouLG20}. Our analysis will show that starting from random initialization the scale of the network grows in a controlled manner which is captured by the following definition.
\begin{definition}\label{def:perpetually_bounded_training}
We say that {\em training is perpetually $\Lambda$ bounded}
if, for all $t\ge 0$ and all $S \subseteq [L]$, $$\prod_{j \in S} \left\lv W_{j}^{(t)}\right\rv_{op}\le \Lambda.$$
\end{definition}
In our subsequent analysis, this notion of perpetually $\Lambda$ bounded shall allow us to control the behavior of the network in the null space of the data matrix $X$. 

\subsection{The Minimum $\ell_2$-norm Interpolant}

It will be helpful to compare the generalization of the
deep linear model with the result of applying the minimum $\ell_2$-norm interpolant resulting from the standard parameterization.
\begin{definition}
\label{d:OLS}
For any $X \in \R^{n \times d}$ and $Y \in \R^{n \times q}$,
define $\Theta_{\ell_2} = X^{\top} (XX^{\top})^{-1}Y$.
\end{definition}
Under Assumption~\ref{assumption:last}, the matrix $XX^\top$ is full rank and therefore $\Theta_{\ell_2}$ is well defined. As previously noted, the excess risk of this canonical interpolator has been studied in prior work~\citep[see][]{bartlett2020benign,tsigler2020benign}.

\section{Main Results}\label{s:main_results}
In this section, we present our excess risk bounds. Our first result applies to any deep linear model trained until interpolation. Second, we shall specialize this result to the case where the model is randomly initialized. Lastly, we present an excess risk bound for a randomly initialized network in a setting with a spiked covariance matrix. 
\subsection{Excess Risk bound for Deep Linear Models}
The following theorem is an excess risk bound for \emph{any} deep linear
model trained until it interpolates in terms of the rate of convergence
of its training, 
along with the effective ranks of the
covariance matrix.
\begin{theorem}
\label{t:generalization.deep}Under Assumptions~\ref{assumption:first}-\ref{assumption:last}, there is an absolute constant $ c > 0$ such that,  for all
$\delta < 1/2$ and all depths $L>1$,
the following holds.
With probability at least $1-c\delta$,
if $\Theta = \lim_{t \rightarrow \infty} \Theta^{(t)}$ for a 
perpetually $\Lambda$-bounded training process for which
$\lim_{t \rightarrow \infty} \cL(\Theta^{(t)}) = 0$, 
and $n \geq c \max\{ r_0,k,\log(1/\delta) \}$,
then
\begin{align*}
\mathsf{Risk}(\Theta) & = 
% \mathsf{Bias(\Theta_{X})}  
\mathsf{Bias}(\Theta_{\ell_2})    
+\mathsf{Variance}(\Theta_{\ell_2})    +\mathsf{\Xi}  ,
\end{align*}
where 
\begin{align*}
    \mathsf{Bias}(\Theta_{\ell_2})     &\le \frac{cs_k}{n} \lv \Theta^\star \rv^2, \\
    \mathsf{Variance}(\Theta_{\ell_2})     & \le  c q \log(q/\delta)\left(\frac{k}{n} + \frac{n}{R_{k}} \right),
\end{align*}
and 
\begin{align*}
    \mathsf{\Xi} = \frac{c s_k}{n}\lv\Theta_{X^{\perp}} \rv^2 & \le 
  \frac{cs_k}{n}\left[ 
     \lv \Theta^{(0)}_{X^{\perp}}\rv
     +
      L
     \Lambda^2 \sqrt{\lambda_1 n}\int_{t=0}^\infty
 \sqrt{\cL(\Theta^{(t)})}\;\mathrm{d}t\right]^2
            .
\end{align*}
\end{theorem}
This bound shows that the \emph{conditional bias} of the estimator $\Theta$ is upper bounded by the conditional bias of the minimum $\ell_2$-norm interpolant $\mathsf{Bias}(\Theta_{\ell_2})$ plus $\mathsf{\Xi}$, which is the additional bias incurred by the component of $\Theta$ outside the row span of $X$. This additional term $\mathsf{\Xi}$ depends not only on the eigenvalues of the covariance matrix but also on the specifics of the optimization procedure such as the initial linear model ($\Theta^{(0)}$), the size of the weights throughout training $(\Lambda)$ and the rate of decay of the loss.

Interestingly, the \emph{conditional variance} of the interpolator $\Theta$, is in fact identical to the conditional variance of the minimum $\ell_2$-norm interpolant.  This follows because, as we will show in the proof, the component of the interpolator $\Theta$ in the row span of $X$ is in fact equal to $\Theta_{\ell_2}$, and the conditional variance depends only on this component within the row span of $X$. The variance captures the effect of perfectly fitting the noise in the data, and our analysis shows that the harm incurred by fitting the noise is unaffected by parameterizing a linear model as a deep linear model. 

Essentially matching lower bounds (up to constants) on the variance term are known \citep{bartlett2020benign}.

\subsection{Excess Risk Bound under Random Initialization}
Our next main result establishes a high-probability bound on
the excess risk, 
and in particular on $\mathsf{\Xi}$,
when the network is trained after a random
initialization (see Definition~\ref{def:random_init}).

\begin{theorem}
\label{t:risk_ab_initialization}
Under Assumptions~\ref{assumption:first}-\ref{assumption:last}, 
there is an absolute constant $c > 0$ such that, 
for all $\delta < 1/2$,
if
\begin{itemize}
    \item the initialization 
    scales
    $\beta$ and $\alpha$ satisfy 
    $\beta \geq 
   c
    \max\left\{ 1,
         \frac{ 
   \lambda_1^{1/4} \sqrt{L n} \left( \sqrt{\lv \Theta^\star \rv} \lambda^{1/4} + q^{1/4}
      \right)
       }{
       \sqrt{s_k}
       }
      \right\}$
      and $\alpha \leq 1$;
    \item the width  $m \ge c \max\left\{  d+q+\log(1/\delta),\frac{ L^2\alpha^2 \lambda_1 s_0 n^2 q\log(n/\delta)}
                     {\beta^2 
                     s_k^2} \right\}$;
    \item the network is trained using random initialization as described in Definition~\ref{s:training};
  \item the number of samples satisfies
 $n \geq c \max\{ r_0,k, \log(1/\delta) \}$,
\end{itemize}
then, with probability at least $1-c\delta$, 
\begin{align*}
\mathsf{Risk}(\Theta) & \le \mathsf{Bias(\Theta_{\ell_2})}    +\mathsf{Variance(\Theta_{\ell_2})}+ \mathsf{\Xi},
\end{align*}
where
\begin{align*}
    \mathsf{Bias(\Theta_{\ell_2})}  &\le \frac{cs_k}{n} \lv \Theta^\star \rv^2, \\
    \mathsf{Variance(\Theta_{\ell_2})} & \le  c q \log(q/\delta)\left(\frac{k}{n} + \frac{n}{R_{k}} \right), \\
   \mathsf{\Xi} 
    & \leq  \frac{c \alpha^2 s_k}{n}\left[q  \beta^2+  
          \frac{L^2(\alpha+1/L)^4\lambda_1 n^2}{s_k^2}\left(\lambda_1 \lv \Theta^\star \rv^2 +q +\frac{\alpha^2 \beta^2 s_0 q \log(n/\delta)}{m}\right)
          \right]. 
\end{align*}
\end{theorem}
Note that the bound on $\mathsf{\Xi}$ of Theorem~\ref{t:risk_ab_initialization}
can be made arbitrarily small by decreasing $\alpha$ while keeping the
other parameters fixed.  When $\alpha = 0$, our bound shows that the model has the same risk as the minimum $\ell_2$-norm interpolant. 

Recall from the simulation in Figure~\ref{fig:linear_models_risk_versus_init_scale}
that as the initialization of the last layer $\beta$ approaches $0$,
the model produced by gradient descent gets closer to the minimum $\ell_2$-norm interpolant. Our bound on $\mathsf{\Xi}$ does not approach $0$ as $\beta \to 0$, and we do not know how to prove that this happens in general with high probability.

Regarding the role of overparameterization, we find that one component of our bound on $\mathsf{\Xi}$ gets smaller as the width $m$ is increased. However, our bound gets larger as we increase depth $L$.  

As mentioned earlier, the bound on the conditional variance, which captures the effect of fitting the noise, is sharp up to constants, however we do not know whether the upper bound on the conditional bias, and specifically $\mathsf{\Xi}$ in
Theorem~\ref{t:risk_ab_initialization}, can be improved.  It is also unclear whether
conditions on $\beta$ and $m$ can be relaxed.

Next, 
to facilitate the interpretation of our bounds,
we apply
Theorem~\ref{t:risk_ab_initialization}
in a canonical
setting where benign overfitting occurs for the minimum $\ell_2$-norm interpolant.
\begin{definition}[$(k,\epsilon)$-spike model]
For $0 < \epsilon < 1$ and $k \in \N$, 
a  $(k,\epsilon)$-spike model is 
a setting
where the eigenvalues of $\Sigma$ are
$\lambda_1 = \ldots = \lambda_k=1$ and
$\lambda_{k+1} = \ldots = \lambda_d = \epsilon$.  
\end{definition}
The $(k,\epsilon)$-spike model is a setting where there are $k$ high variance directions, and many $(d-k)$ low variance directions that can be used to ``hide'' the energy of the noise. Note that, in this model, if $d \geq c n$ and $n \ge ck$ for a large enough
constant $c$, then
$k$ satisfies the requirement of Definition~\ref{def:k_star}, since $r_k =\epsilon(d-k)/\epsilon =d-k \ge bn$.  Since this covariance matrix has full rank,
it may be used in one of the concrete settings where all of our assumptions are
satisfied described at the end of Section~\ref{s:setting}.

\begin{corollary}
\label{c:k_eps_p} 
Under Assumptions~\ref{assumption:first}-\ref{assumption:last}, there is an absolute constant $c > 0$, such that, for any $0 < \epsilon < 1$ and $k \in \N$, 
if $\Sigma$ is an instance of
the $(k, \epsilon)$-spike model,
for any
input dimension $d$, output dimension $q$, depth $L > 1$,
and number of samples $n$,
there are initialization scales $\alpha > 0$ and $\beta > 0$ such that the following holds. For all $\delta < 1/2$,
if
\begin{itemize}
    \item the width $m \ge c (d+q+\log(1/\delta))$;
    \item the network is trained as described in Section~\ref{s:training};
    \item the input dimension  $d \geq c n$;
  \item the number of samples $n \geq c \max\left\{k+\epsilon d,  \log(1/\delta)\right\}$,
\end{itemize}
then, with probability at least $1-c\delta$, 
\begin{align}
\label{e:spike_detailed}
\mathsf{Risk}(\Theta) 
 & \leq 
   c
 \Bigg(
   \frac{\epsilon d \lv \Theta^\star \rv^2 + q k \log(q/\delta)}{n} 
   +
  \frac{n q \log(q/\delta)}{d} \Bigg).
\end{align}
\end{corollary}
Corollary~\ref{c:k_eps_p} gives the simple bound obtained by a choice of parameters
that includes a sufficiently small value of $\alpha$.
For larger values of $\alpha$ the bound
of Theorem~\ref{t:risk_ab_initialization} may behave differently in the
case of the $(k, \epsilon)$-spike model. We find that if we regard $\lv \Theta^\star\rv^2$ as a constant then, the excess risk approaches zero if 
\begin{align*}
    \frac{\epsilon d}{ n} \to 0,\quad  \frac{qk \log(q/\delta)}{n}\to 0\quad  \text{and}\quad \frac{n q}{d} \to 0,
\end{align*}
which recovers the known sufficient conditions for the minimum $\ell_2$-norm interpolant to benignly overfit in this setting.   One example is where
\[
q = 5, k = 5, \delta = 1/100, d = n^2, \epsilon = 1/n^2,
\]
and $n \rightarrow \infty$.

\section{Proof of Theorem~\ref{t:generalization.deep}}\label{s:proof_main_theorem}

The proof of Theorem~\ref{t:generalization.deep} needs some
lemmas, which we prove first.  
Throughout this section the assumptions of Theorem~\ref{t:generalization.deep} are in force.

A key point is that the projection of any interpolator onto
the row span of $X$, including the model output by training
a deep linear network, is the minimum $\ell_2$-norm interpolant.

\begin{lemma}
\label{l:OLS_in_span}
For any interpolator $\Theta$, 
$\Theta_X = P_{X} \Theta =
\Theta_{\ell_2}$.
\end{lemma}
\begin{proof} Since $\Theta$ interpolates the data
\begin{align}
    Y = X \Theta 
     = X \left(\Theta_{X} + \Theta_{X^{\perp}}\right) 
     = X \Theta_{X} . \label{e:y_interpolation_condition}
\end{align}
Recall that $\Theta_X = X^{\top}(XX^{\top})^{-1}X\Theta = P_{X} \Theta$, where $P_{X}$ projects onto the row span of $X$. Continuing, we get that
\begin{align*}
    \Theta_{\ell_2} 
    &= X^{\top}(XX^{\top})^{-1}Y \\ 
    &= X^{\top}(XX^{\top})^{-1}X \Theta_{X} \\ & = P_{X} \Theta_{X}  = P_{X} P_{X} \Theta  = P_{X} \Theta  = \Theta_{X}.
\end{align*}
\end{proof}

Using the formula for the minimum $\ell_2$-norm interpolant,
we can now write down an expression
for the excess risk.
\begin{lemma}
\label{l:risk_decomposition}
The excess risk of any interpolator $\Theta$ of the data satisfies
\begin{align*}
    \mathsf{Risk}(\Theta) & \le c \Tr\left((\Theta^\star -\Theta_{X^{\perp}})^{\top}B(\Theta^\star -\Theta_{X^{\perp}})\right)+cq\log(q/\delta)\Tr(C)
\end{align*}
with probability at least $1-\delta$ over the noise matrix $\Omega = Y-X\Theta^\star$, where
\begin{align*}
    B &:= \left(I - X^{\top}(XX^{\top})^{-1}X\right)\Sigma\left(I - X^{\top}(XX^{\top})^{-1}X\right) \quad \text{and} \\
    C &:= (XX^{\top})^{-1}X\Sigma X^{\top}(XX^{\top})^{-1}.
\end{align*}
\label{l:excess_risk_decomposition}
\end{lemma}
\begin{proof}
We have $y - x\Theta^\star$ is conditionally mean-zero given $x$, thus
\begin{align*}
    \mathsf{Risk}(\Theta) &= \E_{x,y}\left[\lv y - x\Theta\rv^2\right]-\E_{x,y}\left[\lv y - x\Theta^{\star}\rv^2\right]\\
    &= \E_{x,y}\left[\lv y - x\Theta^{\star}+x(\Theta^\star -\Theta)\rv^2\right]-\E_{x,y}\left[\lv y - x\Theta^{\star}\rv^2\right]\\
    & = \E_{x}\left[\lv x(\Theta^\star -\Theta)\rv^2\right]. \numberthis \label{e:risk_equality_expectation}
\end{align*}
Since $\Theta$ interpolates the data, by Lemma~\ref{l:OLS_in_span} we know that 
\[
\Theta = \Theta_{\ell_2}+\Theta_{X^{\perp}} = X^\top(XX^\top)^{-1}Y +\Theta_{X^{\perp}}.
\]
Now because $Y = X\Theta^\star +\Omega$ we find that
\begin{align*}
    & \mathsf{Risk}(\Theta) \\
    & = \E_x\left[\left\lv x\left(I-X^{\top}(XX^{\top})^{-1}X\right)(\Theta^\star -\Theta_{X^{\perp}})-x X^\top(XX^\top)^{-1}\Omega\right\rv^2\right]\\
    & \le 2\E_x\left[\left\lv x\left(I-X^{\top}(XX^{\top})^{-1}X\right)(\Theta^\star -\Theta_{X^{\perp}})\right\rv^2\right]
    +2\E_{x}\left[\left\lv x X^\top(XX^\top)^{-1}\Omega\right\rv^2\right]\\
    & \le 2\E_x\left[ x\left(I-X^{\top}(XX^{\top})^{-1}X\right)(\Theta^\star -\Theta_{X^{\perp}})(\Theta^\star -\Theta_{X^{\perp}})^{\top}\left(I-X^{\top}(XX^{\top})^{-1}X\right)x^{\top}\right] \\
    & \qquad +2\E_{x}\left[
    x X^\top(XX^\top)^{-1}\Omega 
    \Omega^{\top} (XX^\top)^{-1} X x^{\top}
     \right]\\
     & \overset{(i)}{=} 2 \Tr\left((\Theta^\star \!-\!\Theta_{X^{\perp}})^{\top}\left(I-X^{\top}(XX^{\top})^{-1}X\right)\E_x\left[x^{\top}x\right]\left(I-X^{\top}(XX^{\top})^{-1}X\right)(\Theta^\star \!-\!\Theta_{X^{\perp}})\right) \\
    & \qquad +2\Tr \left(
    \Omega^{\top} (XX^\top)^{-1} X \E_{x} [x^{\top} x] 
     X^\top(XX^\top)^{-1}\Omega
     \right)\\
     & \overset{(ii)}{=} 2 \Tr\left((\Theta^\star -\Theta_{X^{\perp}})^{\top}B(\Theta^\star -\Theta_{X^{\perp}})\right)  +2\Tr \left(
    \Omega^{\top} C \Omega
     \right).
\end{align*}
where $(i)$ follows by using the cyclic property of the trace, and $(ii)$ follows by the definition of the matrices $B$ and $C$.

Let $\omega_1,\ldots,\omega_q$ denote the columns of the error matrix $\Omega$. Then
\begin{align*}
    \Tr(\Omega^{\top}C \Omega) = \sum_{i=1}^q \omega_i^{\top}C \omega_i.
\end{align*}
Invoking \citep[][Lemma~S.2]{bartlett2020benign} bounds each term in the sum by $c q \log(q/\delta)\Tr(C)$ with probability at least $1-\delta/q$. A union bound completes the proof.
\end{proof}

To work on the first term in the upper bound of the excess risk, we would like an upper bound on $\lv \Theta^{(t)}_{X^{\perp}} \rv$.  Toward this end, we first establish
a high-probability bound on $\lv X \rv_{op}$.
\begin{lemma}
\label{l:Xop_bound}
There is a constant $c > 0$ such that for any $\delta \in (0,1)$, if
$
n \!\geq\! c \max \{r_0, \log\left(\frac{1}{\delta}\right) \},
$
with probability at least $1 - \delta$, $\lv X \rv_{op} \leq c\sqrt{\lambda_1 n}$.
\end{lemma}
\begin{proof}
By \citep[][Lemma~9]{koltchinskii2017concentration}, with probability at least $1 - \delta$ 
\begin{align*}
\lv X \rv_{op}
  & = \sqrt{\lv X^{\top} X \rv_{op}} \\
  & \leq \sqrt{n \left( \lv \Sigma \rv_{op} 
               + \left\lv \frac{1}{n} X^{\top} X - \Sigma \right\rv_{op} \right) } \\
  & \leq
  \sqrt{n \left( \lv \Sigma \rv_{op} 
               + \lv \Sigma \rv_{op} \max\left\{
                        \sqrt{\frac{r_0}{n}},
                        \sqrt{\frac{\log(1/\delta)}{n}},
                        \frac{r_0}{n},
                        \frac{\log(1/\delta)}{n}
                         \right\} \right) }.
\end{align*}
 Recalling that $n \geq c \max\{ r_0, \log(1/\delta) \}$, this implies
that, with probability at least $1 - \delta$, $\lv X \rv_{op} \leq c\sqrt{n \lv \Sigma \rv_{op}}$.
\end{proof}

%LMC We also calculate a formula for the time derivative of $\Theta^{(t)}$ in the lemma below.
Next, we will calculate a formula for the time derivative of $\Theta^{(t)}$.
Its definition will make use of products of matrices before and after
a given layer.
\begin{definition}
\label{d:AB}
For $j \in [L]$ define 
$A_j = \prod_{k=L}^{j+1} W_k^{(t)}$ and $B_j = \prod_{k=j-1}^{1} W_k^{(t)}$.
\end{definition}
Now we are ready for our lemma giving the time derivative of $\Theta^{(t)}$.
\begin{lemma}\label{l:theta_time_derivative}
At any time $t\ge 0$,
\begin{align*}
    \frac{\mathrm{d} \Theta^{(t)}}{\mathrm{d}t } & =  -\sum_{j=1}^L 
      B_j^{\top}
      B_j
       X^{\top}
      (X\Theta^{(t)} - Y)
      A_j
      A_j^{\top}.
\end{align*}

\end{lemma}
\begin{proof} Let us suppress the superscript $(t)$ to ease notation.
The gradient flow dynamics is defined as 
\begin{align*}
    \frac{\mathrm{d}W_{j}}{\mathrm{d} t} & = - \nabla_{W_{j}} \cL(\Theta),
\end{align*}
where
\begin{align}
\nabla_{W_{j}} \cL(\Theta)
 = (W_L \cdots W_{j+1})^{\top}(X\Theta - Y)^{\top}\left(W_{j-1}\cdots W_1 X^{\top}\right)^{\top}. \label{e:gradient_formula}
\end{align}

So by the chain rule of differentiation,
\begin{align*}
    \frac{\mathrm{d} \Theta}{\mathrm{d} t} & = \frac{\mathrm{d} \left(W_1^{\top}\cdots W_{L}^{\top}\right)}{\mathrm{d} t} \\
    & = \sum_{j=1}^L \left(W_{1}^{\top}\cdots W_{j-1}^{\top} \right) \frac{\mathrm{d}W_j^{\top}}{\mathrm{d} t} \left(W_{j+1}^{\top}\cdots W_{L}^{\top}\right) \\
    & = \sum_{j=1}^L \left(W_{1}^{\top}\cdots W_{j-1}^{\top} \right) \left(-\nabla_{W_j}\cL(\Theta)\right)^{\top} \left(W_{j+1}^{\top}\cdots W_{L}^{\top}\right) \\
    & = -\sum_{j=1}^L 
      \left(W_{1}^{\top}\cdots W_{j-1}^{\top} \right) 
       \left(W_{j-1}\cdots W_1 X^{\top}\right)
      (X\Theta - Y)
      (W_L \cdots W_{j+1})
      \left(W_{j+1}^{\top}\cdots W_{L}^{\top}\right) \\
    & = -\sum_{j=1}^L 
      B_j^{\top}
      B_j
       X^{\top}
      (X\Theta - Y)
      A_j
      A_j^{\top}.
\end{align*}

\end{proof}

Toward the goal of proving a high-probability bound on
$\lv \Theta^{(t)}_{X^{\perp}} \rv$, we next bound its rate
of growth.
\begin{lemma}
\label{lem:outside_span_norm_growth_bound} 
There is a constant $c > 0$ such that, if
$n \geq c \max \{r_0, \log(1/\delta) \}$,
with probability at least $1 - \delta$, 
if training is perpetually $\Lambda$ bounded, then, for all $t\ge 0$, 
\begin{align*}
    \frac{1}{2}\frac{\mathrm{d}\lv \Theta^{(t)}_{X^{\perp}}\rv^2}{\mathrm{d}t} 
 & \le 
    \sum_{j=2}^L \Tr
\left(
\Theta^{(t)\top} 
  P_{X^{\perp}}
      B_j^{\top}
      B_j
       X^{\top}
      (X\Theta^{(t)} - Y) 
%       \right. \\
%  & \hspace{4in}
%      \left. 
     A_j
      A_j^{\top}
  \right). 
\end{align*}
\end{lemma}
\begin{proof}Given matrices $A$ and $B$, we let $A\cdot B = \Tr(A^{\top}B)$ denote the matrix inner product.

By the chain rule,
\begin{align*}
   \frac{1}{2}\frac{\mathrm{d}\lv \Theta^{(t)}_{X^{\perp}}\rv^2}{\mathrm{d}t} 
  & =\Theta_{X^{\perp}}^{(t)}\cdot \frac{\mathrm{d}\Theta_{X^{\perp}}^{(t)}}{\mathrm{d}t} \\
  & = \Theta_{X^{\perp}}^{(t)}\cdot P_{X^{\perp}}\frac{\mathrm{d}\Theta^{(t)}}{\mathrm{d}t} \\
  & \overset{(i)}{=} \Theta_{X^{\perp}}^{(t)}\cdot P_{X^{\perp}}
  \Bigg(
  -\sum_{j=1}^L 
      B_j^{\top}
      B_j
       X^{\top}
      (X\Theta^{(t)} - Y) 
      % \\&\hspace{3.8in}
      A_j
      A_j^{\top}
  \Bigg)
  \\
  & = -\sum_{j=1}^L  \Theta_{X^{\perp}}^{(t)}\cdot P_{X^{\perp}}
  \Biggl(
      B_j^{\top}
      B_j
       X^{\top}
      (X\Theta^{(t)} - Y) 
      % \\&\hspace{3.8in}
      A_j
      A_j^{\top}
 \Biggr),\label{e:change_in_theta_perp}\numberthis
\end{align*}
where $(i)$ follows by the formula derived in Lemma~\ref{l:theta_time_derivative}.

Let us consider a particular term in the sum above,
\begin{align*}
& 
\Theta_{X^{\perp}}^{(t)}\cdot P_{X^{\perp}}
  \left(
      B_j^{\top}
      B_j
       X^{\top}
      (X\Theta^{(t)} - Y)
      A_j
      A_j^{\top}
  \right) \\
& 
=\Tr
\Bigg(
\Theta_{X^{\perp}}^{(t)\top} P_{X^{\perp}}
      B_j^{\top}
      B_j
       X^{\top}
      (X\Theta^{(t)} - Y) 
      % \\&\hspace{3.8in}
      A_j
      A_j^{\top}
  \Bigg) \\
& 
=\Tr
\Bigg(
\Theta^{(t)\top} 
P_{X^{\perp}}^{\top}
  P_{X^{\perp}}
      B_j^{\top}
      B_j
       X^{\top}
      (X\Theta^{(t)} - Y) 
      % \\&\hspace{3.8in}
      A_j
      A_j^{\top}
  \Bigg)  \\
& 
=\Tr
\Biggl(
\Theta^{(t)\top} 
  P_{X^{\perp}}
      B_j^{\top}
      B_j
       X^{\top}
      (X\Theta^{(t)} - Y)
      % \\&\hspace{3.8in}
      A_j
      A_j^{\top}
  \Biggr)\numberthis \label{e:change_in_theta_perp_more}.
\end{align*}
In the case where $j=1$, the RHS is equal to
\begin{align*}
    \Tr
\left(
\Theta^{(t)\top} 
  P_{X^{\perp}}
       X^{\top}
      (X\Theta^{(t)} - Y)
      \left(\prod_{k=L}^{2} W_k^{(t)} \right)
      \left(\prod_{k=L}^{2} W_k^{(t)} \right)^{\top}
  \right) = 0,
\end{align*}
since $P_{X^{\perp}} X^{\top} = 0$.

In the case $j > 1$, we have
\begin{align*}
& 
\Theta_{X^{\perp}}^{(t)}\cdot P_{X^{\perp}}
  \left(
      B_j^{\top}
      B_j
       X^{\top}
      (X\Theta^{(t)} - Y)
      A_j
      A_j^{\top}
  \right)\\
& 
=\Tr
\left(
\Theta^{(t)\top} 
  P_{X^{\perp}}
      B_j^{\top}
      B_j
       X^{\top}
      (X\Theta^{(t)} - Y)
      A_j
      A_j^{\top}
  \right) 
\end{align*}
completing the proof.
\end{proof}

\begin{lemma}
\label{lem:outside_span_norm_bound} 
There is a constant $c > 0$ such that, if
$n \geq c \max \{r_0, \log(1/\delta) \}$,
with probability at least $1 - \delta$, 
if training is perpetually $\Lambda$ bounded, then, for all $t\ge 0$, 
\begin{align*}
    \lv \Theta_{X^{\perp}}^{(t)}\rv \le
    \lv \Theta^{(0)}_{X^{\perp}}\rv+c   (L-1) 
     \Lambda^2\sqrt{\lambda_1 n}\int_{s=0}^t 
 \sqrt{\cL(\Theta^{(s)})}\;\mathrm{d}s.
\end{align*}
\end{lemma}
\begin{proof} 
Let us consider one of the terms in the RHS of Lemma~\ref{lem:outside_span_norm_growth_bound}.
We have
\begin{align*}
& \Tr
\left(
\Theta^{(t)\top} 
  P_{X^{\perp}}
      B_j^{\top}
      B_j
       X^{\top}
      (X\Theta^{(t)} - Y)
      A_j
      A_j^{\top}
  \right) \\
  & 
=\Tr
\left(
\Theta^{(t)\top}_{X^{\perp}}
      B_j^{\top}
      B_j
       X^{\top}
      (X\Theta^{(t)} - Y)
      A_j
      A_j^{\top}
  \right)\\
    & 
\leq
\lv\Theta^{(t)}_{X^{\perp}}\rv
      \left\lv B_j^{\top}
      B_j
       X^{\top}
      (X\Theta^{(t)} - Y)
      A_j
      A_j^{\top}\right\rv \\
& \overset{(i)}{\le} \lv\Theta^{(t)}_{X^{\perp}}\rv \left\lv B_j^{\top}
      B_j\right\rv_{op} 
\left\lv A_j
      A_j^{\top}\right\rv_{op}
\left\lv X \right\rv_{op} \left\lv X\Theta^{(t)}-Y\right\rv \\
& \le \lv\Theta^{(t)}_{X^{\perp}}\rv \left(\prod_{k=1}^{j-1} \lv W_k^{(t)} \rv_{op}^2\right) \left(\prod_{k=j+1}^{L} 
\lv  W_k^{(t)}\rv_{op}^2\right)
\left\lv X \right\rv_{op} \left\lv X\Theta^{(t)}-Y\right\rv \\
& = \lv\Theta^{(t)}_{X^{\perp}}\rv \left(\prod_{k\neq j} \lv W_k^{(t)} \rv_{op}\right)^2 
\left\lv X \right\rv_{op} \left\lv X\Theta^{(t)}-Y\right\rv \\
& \overset{(ii)}{\le} \lv\Theta^{(t)}_{X^{\perp}}\rv \left\lv X \right\rv_{op} 
\Lambda^2
 \sqrt{\cL(\Theta^{(t)})},
\end{align*}
where $(i)$ follows since for any matrices $\lv AB\rv \le \lv A\rv_{op}\lv B\rv$, and $(ii)$ follows since 
training is perpetually $\Lambda$ bounded.

Summing over layers $j = 2,\ldots,L$, we get that,
\begin{align*}
    \frac{1}{2}\frac{\mathrm{d}\lv \Theta^{(t)}_{X^{\perp}}\rv^2}{\mathrm{d}t}
     \le (L-1)\lv\Theta^{(t)}_{X^{\perp}}\rv \left\lv X \right\rv_{op} \Lambda^2
 \sqrt{\cL(\Theta^{(t)})}.
\end{align*}

% Now consider two cases. 

% \textbf{Case 1:} ($\lv \Theta^{(0)}_{X^{\perp}}\rv = 0$) Here, the bound on the time derivative above implies that $\lv \Theta^{(t)}_{X^{\perp}}\rv = 0$ for all $t\ge 0$. 

% \textbf{Case 2:} ($\lv \Theta^{(0)}_{X^{\perp}}\rv \neq 0$)
% In this case,
Now note that,
\begin{align*}
        \frac{1}{2}\frac{\mathrm{d}\lv \Theta^{(t)}_{X^{\perp}}\rv^2}{\mathrm{d}t} & = \frac{\lv \Theta^{(t)}_{X^{\perp}}\rv\mathrm{d}\lv \Theta^{(t)}_{X^{\perp}}\rv}{\mathrm{d}t} 
     \le (L-1)\lv\Theta^{(t)}_{X^{\perp}}\rv \left\lv X \right\rv_{op} \Lambda^2
 \sqrt{\cL(\Theta^{(t)})},
\end{align*}
which in turn implies that, when $\lv \Theta^{(t)}_{X^{\perp}}\rv > 0$, we
have
\begin{align*}
         \frac{\mathrm{d}\lv \Theta^{(t)}_{X^{\perp}}\rv}{\mathrm{d}t} 
     \le (L-1) \left\lv X \right\rv_{op} \Lambda^2
 \sqrt{\cL(\Theta^{(t)})}.
\end{align*}
If, for all $s \in [0,t]$, we have $\lv \Theta^{(t)}_{X^{\perp}}\rv \neq 0$,
then by integrating this differential inequality we conclude that
\begin{align}
\label{e:bound}
    \lv \Theta_{X^{\perp}}^{(t)}\rv-\lv \Theta_{X^{\perp}}^{(0)}\rv & \le (L-1) \left\lv X \right\rv_{op} \Lambda^2\int_{s=0}^t 
 \sqrt{\cL(\Theta^{(s)})}\;\mathrm{d}s.
\end{align}
Otherwise, if $T = \sup \{ s : \lv \Theta^{(s)}_{X^{\perp}}\rv = 0 \}$,
\begin{align*}
    \lv \Theta_{X^{\perp}}^{(t)}\rv & \le (L-1) \left\lv X \right\rv_{op} \Lambda^2\int_{s=T}^t 
 \sqrt{\cL(\Theta^{(s)})}\;\mathrm{d}s,
\end{align*}
which implies \eqref{e:bound}.

Applying Lemma~\ref{l:Xop_bound} which is a high probability upper bound on $\lv X\rv_{op}$ completes the proof.
\end{proof}

Armed with these lemmas, we are now ready to prove the
first of our main results.

\begin{proof}[Proof of Theorem~\ref{t:generalization.deep}]
Combining Lemma~\ref{l:risk_decomposition} with 
Lemmas 6 and 11 by \citet{bartlett2020benign} to bound $\Tr(C)$ we get,
with probability at least $1 - c \delta$,
\begin{align*}
    \mathsf{Risk}(\Theta) & \le \Tr\left((\Theta^\star -\Theta_{X^{\perp}})^{\top}B(\Theta^\star -\Theta_{X^{\perp}})\right)  + \underbrace{c q \log(q/\delta)\left(\frac{k}{n} + \frac{n}{R_{k}} \right)}_{=: \mathsf{Variance}(\Theta_{\ell_2})}.
\end{align*}

We begin by bounding the first term in the RHS above. Let
$\theta^\star_1,\ldots, \theta^\star_q$ be the columns of
$\Theta^\star$ and $\theta_{X^{\perp},1},\ldots, \theta_{X^{\perp},q}$ be the columns of
$\Theta_{X^{\perp}}$. By invoking~\citep[][Eq.~54]{chatterji2022interplay} for each
of the $q$ outputs, and applying a union bound, we find that with probability at least $1 - c q(\delta/q) = 1 - c \delta$,
\begin{align*}
\Tr\left((\Theta^\star -\Theta_{X^{\perp}})^{\top}B(\Theta^\star -\Theta_{X^{\perp}})\right) 
 & = \sum_{i=1}^q (\theta^{\star}_i - \theta_{X^{\perp},i})^{\top} B (\theta^{\star}_i - \theta_{X^{\perp},i}) \\
 & \le \frac{c s_k}{n} \sum_{i=1}^q \lv \theta^{\star}_i - \theta_{X^{\perp},i}\rv^2 \\
 & = \frac{cs_k}{n}\lv \Theta^\star - \Theta_{X^{\perp}}\rv^2  \\
 & \le \frac{2cs_k}{n}\left(\lv \Theta^\star \rv^2+\lv \Theta_{X^{\perp}}\rv^2  \right).
\end{align*}
Define $\mathsf{Bias(\Theta_{\ell_2})}  := \frac{2cs_k}{n}\lv \Theta^\star\rv^2$ and let $\mathsf{\Xi}  :=  \frac{2cs_k}{n}\lv \Theta_{X^{\perp}}\rv^2$. The bound on $\mathsf{\Xi}$ follows by invoking Lemma~\ref{lem:outside_span_norm_bound}.
\end{proof}

\section{Proof of Theorem~\ref{t:risk_ab_initialization}}\label{s:proof_random_init_theorem}

The assumptions of Theorem~\ref{t:risk_ab_initialization}
are in force throughout this section.  Before starting
its proof, we establish some lemmas.

\begin{restatable}{definition}{initdef}\label{def:good_ab_initialization}
For a large enough absolute constant $c$,
we say that the network enjoys a $\delta$-\emph{good initialization} if 
\begin{align*}
    \alpha/c < \sigma_{\min}(W_{1}^{(0)}) &\le  \sigma_{\max}(W_{1}^{(0)}) < c \alpha ,
      \\
  \beta/c  < \sigma_{\min}(W_{L}^{(0)}) &\le \sigma_{\max}(W_{L}^{(0)}) < c \beta,
\end{align*}
and
\begin{align*}
    \cL(\Theta^{(0)}) < c \left(\lv Y\rv^2 + \frac{\alpha^2 \beta^2 q \lv X\rv^2\log(n/\delta)}{m} \right).
\end{align*}
\end{restatable}
The following proposition is proved in Appendix~\ref{a:optimization_ab_initialization}. It guarantees that for wide networks, optimization is successful starting from random initialization. 
\begin{restatable}{proposition}{optprop}
\label{p:optimization_ab_initialization} There 
is a constant $c$
such that, given any $\delta \in (0,1)$, if
    the initialization scales $\alpha$ and $\beta$, along with
        the network width $m$, satisfy 
        \begin{align*}
          &  m \ge c\max  \left\{  d+q+\log(1/\delta), \frac{L^2 \alpha^2 \lv X\rv_{op}^2 \lv X\rv^2 q\log(n/\delta)}
                     {\beta^2 \sigma^4_{\min}(X)} \right\}, \\
          & \beta \geq c \max\left\{1, \sqrt{\frac{L \lv X\rv_{op}\lv Y\rv}
          { \sigma^2_{\min}(X)}}  \right\}, \\
          & \alpha \leq 1,
        \end{align*}
then with probability at least
$1 - \delta$:
\begin{enumerate}[leftmargin=0.6in]
\item the initialization is $\delta$-good;
\item training is perpetually $c (\alpha + 1/L) \beta $ bounded;
    \item for all $t > 0$, we have that
\begin{align*}
   \cL(\Theta^{(t)})
   & < \cL(\Theta^{(0)})\exp\left(- \frac{\beta^2 \sigma_{\min}^2(X) }{4e}\cdot t \right) \\
    &  \leq
    c \left(\lv Y\rv^2 + \frac{\alpha^2 \beta^2 q\lv X\rv^2 \log(n/\delta)}{m} \right)
    \exp\left(- \frac{\beta^2 \sigma_{\min}^2(X) }{4e}\cdot t \right).
\end{align*}
\end{enumerate}
\end{restatable}
The reader may notice that the roles
of $\alpha$ and $\beta$ in Proposition~\ref{p:optimization_ab_initialization}
are asymmetric.  
We focused on that case that $\alpha$ is small
because
the updates of $W_1$
are in the span of the rows of $X$, which is not necessarily
the case for the other layers, including $W_L$.
This means that the scale of $W_1$ in the null space of $X$ remains the same as it was at initialization,
so that a small scale at initialization pays dividends throughout training.

The next lemma shows that the projection of the model computed by
the network onto the null space of $X$ is the
same as the model obtained by projecting the first layer weights, and combining them with the other layers.

\begin{lemma}
\label{l:Theta_perp.Wone_perp}
For all $t \ge 0$,
\begin{align*}
\Theta_{X^{\perp}}^{(t)} 
    =  \left( W_L^{(t)} \cdots W_2^{(t)} W_{1, X^\perp}^{(t)} \right)^{\top},
\end{align*}
where
\[
W_{1, X^\perp}^{(t)} :=  W_1^{(t)} (I - P_X).
\]
\end{lemma}
\begin{proof}By definition
\begin{align*}
    \Theta^{(t)} = (W_L^{(t)} \cdots W_1^{(t)})^\top = \left(W_1^{(t)}\right)^{\top} \cdots \left( W_L^{(t)} \right)^{\top} \in \mathbb{R}^{d\times q}.
\end{align*}
Therefore,
\begin{align*}
    \Theta_{X^\perp}^{(t)} = (I-P_X)\Theta^{(t)}= (I-P_X)(W_1^{(t)})^{\top} \cdots (W_L^{(t)})^{\top}  = (W_{1, X^\perp}^{(t)})^{\top} \cdots (W_L^{(t)})^{\top}.
\end{align*}
\end{proof}
The subsequent lemma shows that the projection of the first layer onto the null space of $X$ does not change during training.
\begin{lemma}
\label{l:Wone_perp_fixed}
For all $t\ge 0$,
$W_{1, X^\perp}^{(t)} =  W_{1, X^\perp}^{(0)}.$
\end{lemma}
\begin{proof}
We have
\begin{align*}
 \frac{\mathrm{d} W_{1, X^\perp}^{(t)}}{\mathrm{d} t} 
& = \frac{\mathrm{d} W_{1}^{(t)} (I-P_X)}{\mathrm{d} t} \\
& = \left( \frac{\mathrm{d} W_{1}^{(t)} }{\mathrm{d} t} \right) (I-P_X) \\
& = -\left( (W_L \cdots W_{2})^{\top}(X\Theta - Y)^{\top} X \right) (I-P_X) && \mbox{(by using Eq.~\eqref{e:gradient_formula})}\\
& = -\left( (W_L \cdots W_{2})^{\top}(X\Theta - Y)^{\top} \right) \left( X  (I-P_X) \right) \\
& = -\left( (W_L \cdots W_{2})^{\top}(X\Theta - Y)^{\top} \right) \left( 0 \right) \\
& = 0.
\end{align*}
\end{proof}
By using the previous two lemmas regarding the first layer weights $W_1$ we can now prove an alternate bound on $\lv \Theta_{X^{\perp}}^{(t)}\rv$. In contrast to the previous bound that we derived in Lemma~\ref{lem:outside_span_norm_bound}, here the initial scale of $W_1$ plays a role in controlling the growth in $\lv \Theta_{X^\perp}^{(t)}\rv$.
\begin{lemma}
There is constant $c > 0$ such that, 
if training is perpetually $\Lambda$ bounded, then, for all $t \ge 0$, 
\begin{align*}
 \lv \Theta^{(t)}_{X^{\perp}}\rv
  \leq
    \lv \Theta^{(0)}_{X^{\perp}}\rv+c  L
     \lv W_1^{(0)}\rv_{op} \lv X\rv_{op} \Lambda^2\int_{s=0}^t 
 \sqrt{\cL(\Theta^{(s)})}\;\mathrm{d}s.
\end{align*}
\end{lemma}
\begin{proof}
Let us once again consider one of the terms in the RHS of Lemma~\ref{lem:outside_span_norm_growth_bound}.
We have
\begin{align*}
& 
\Tr
\left(
\Theta^{(t)\top}_{X^\perp} 
 P_{X^{\perp}}
      B_j^{\top}
      B_j
       X^{\top}
      (X\Theta^{(t)} - Y)
      A_j
      A_j^{\top}
  \right) \\
  & 
=\Tr
\left(
\Theta^{(t)\top} P_{X^\perp} (W_1^{(t)})^{\top}
      \left(\prod_{k=j-1}^{2} W_k^{(t)} \right)^{\top}
      B_j
       X^{\top}
      (X\Theta^{(t)} - Y)
      A_j
      A_j^{\top}
  \right)\\
  & 
=\Tr
\left(
\Theta^{(t)\top}_{X^\perp} W_{1,X^\perp}^{\top(t)}
      \left(\prod_{k=j-1}^{2} W_k^{(t)} \right)^{\top}
      B_j
      X^{\top}
      (X\Theta^{(t)} - Y)
      A_j
      A_j^{\top}
  \right).
  \end{align*}
  Continuing by using the fact that for any matrices $\lv AB \rv \le \lv A\rv_{op}\lv B\rv$, we get that
  \begin{align*}
  & 
\Theta_{X^{\perp}}^{(t)}\cdot P_{X^{\perp}}
  \left(
      B_j^{\top}
      B_j
       X^{\top}
      (X\Theta^{(t)} - Y)
      A_j
      A_j^{\top}
  \right)\\
& \le \lv\Theta^{(t)}_{X^{\perp}}\rv \lv W_{1,X^\perp}^{(t)}\rv_{op} \left\lv\left(\prod_{k=j-1}^{2} W_k^{(t)} \right)^{\top}
      B_j\right\rv_{op} 
\left\lv A_j
      A_j^{\top}\right\rv_{op}
\left\lv X \right\rv_{op} \left\lv X\Theta^{(t)}-Y\right\rv \\
& \overset{(i)}{\le} \lv\Theta^{(t)}_{X^{\perp}}\rv \lv W_{1,X^\perp}^{(t)}\rv_{op} \left\lv X \right\rv_{op} 
\Lambda^2
 \sqrt{\cL(\Theta^{(t)})}\\
 & \overset{(ii)}{\le}  \lv\Theta^{(t)}_{X^{\perp}}\rv \lv W_{1,X^\perp}^{(0)}\rv_{op} \left\lv X \right\rv_{op} 
\Lambda^2
 \sqrt{\cL(\Theta^{(t)})} \\
%LMC reorganized because of overfull hbox
 & \le \lv\Theta^{(t)}_{X^{\perp}}\rv \lv W_1^{(0)}\rv_{op} \lv X\rv_{op}
 \Lambda^2
 \sqrt{\cL(\Theta^{(t)})} ,
\end{align*}
since $\lv W_{1,X^\perp}^{(0)}\rv_{op}\le \lv W_{1}^{(0)}\rv_{op}$,
where $(i)$ follows since 
training is $\Lambda$ perpetually bounded and so
\begin{align*}
    &\left\lv\left(\prod_{k=j-1}^{2} W_k^{(t)} \right)^{\top}
      B_j\right\rv_{op} 
\left\lv A_j
      A_j^{\top}\right\rv_{op} % \\&\qquad 
      \le \left(\prod_{k\neq \{1,j\}} \lv W_k^{(t)}\rv_{op} \right)\left(\prod_{k\neq \{j\}} \lv W_k^{(t)}\rv_{op} \right) % \\ &\qquad 
      \le \Lambda^2
\end{align*}
and (ii) follows since by Lemma~\ref{l:Wone_perp_fixed}, $W_{1,X^\perp}^{(t)}=W_{1,X^\perp}^{(0)}$.

Summing over layers $j = 2,\ldots,L$, we get that,
\begin{align*}
    \frac{1}{2}\frac{\mathrm{d}\lv \Theta^{(t)}_{X^{\perp}}\rv^2}{\mathrm{d}t}
     \le (L-1)\lv\Theta^{(t)}_{X^{\perp}}\rv \lv W_{1}^{(0)}\rv_{op}\left\lv X \right\rv_{op} \Lambda^2
 \sqrt{\cL(\Theta^{(t)})}.
\end{align*}

Thus, we have that 
\begin{align*}
        \frac{1}{2}\frac{\mathrm{d}\lv \Theta^{(t)}_{X^{\perp}}\rv^2}{\mathrm{d}t} & = \frac{\lv \Theta^{(t)}_{X^{\perp}}\rv\mathrm{d}\lv \Theta^{(t)}_{X^{\perp}}\rv}{\mathrm{d}t} 
     \le (L-1)\lv\Theta^{(t)}_{X^{\perp}}\rv \lv W_{1}^{(0)}\rv_{op}\left\lv X \right\rv_{op} \Lambda^2
 \sqrt{\cL(\Theta^{(t)})}
\end{align*}
which in turn implies that, when 
$\lv \Theta^{(t)}_{X^{\perp}}\rv \neq 0$, we have
\begin{align*}
         \frac{\mathrm{d}\lv \Theta^{(t)}_{X^{\perp}}\rv}{\mathrm{d}t} 
     \le (L-1) \lv W_{1}^{(0)}\rv_{op}\left\lv X \right\rv_{op}  \Lambda^2
 \sqrt{\cL(\Theta^{(t)})}.
\end{align*}
Therefore, by integrating this differential inequality 
%LMC
as in the proof of Lemma~\ref{lem:outside_span_norm_bound}, we conclude that
\begin{align*}
    \lv \Theta_{X^{\perp}}^{(t)}\rv-\lv \Theta_{X^{\perp}}^{(0)}\rv & \le (L-1) \lv W_{1}^{(0)}\rv_{op}\left\lv X \right\rv_{op}  \Lambda^2\int_{s=0}^t 
 \sqrt{\cL(\Theta^{(s)})}\;\mathrm{d}s.
\end{align*}
\end{proof}
We also need a lemma that bounds the Frobenius norm of the data matrix $X$.
\begin{lemma}
\label{l:X_fro_bound}
There is a constant $c > 0$ such that for any $\delta \in (0,1)$, if
$n \geq c  \log(1/\delta)$, then with probability at least $1 - \delta$, $\lv X \rv \leq c\sqrt{n s_0}$.
\end{lemma}
\begin{proof}The rows of $X$ are $n$ i.i.d. draws from a distribution, where each sample can be written as $x_i = \Sigma^{1/2} u_i$, where $u_i$ has components that are independent $c_x$-sub-Gaussian random variables. Define $u_{\text{stacked}} := (u_1, u_2, \ldots, u_n)  \in \mathbb{R}^{ dn}$ to be concatenation of the vectors $u_1,\ldots,u_n$ and define $\Sigma^{1/2}_{\text{stacked}} \in \mathbb{R}^{dn \times dn}$ to be a block diagonal matrix with $\Sigma^{1/2} \in \mathbb{R}^{d\times d}$ repeated $n$ times along its diagonal. Then,
\begin{align*}
    \lv X\rv^2  & = \sum_{i=1}^n \lv x_i \rv^2  = \sum_{i=1}^n \lv \Sigma^{1/2} u_i \rv^2  = \left\lv \Sigma^{1/2}_{\text{stacked}} u_{\text{stacked}}
    \right\rv^2.
\end{align*}
Now, $u_{\text{stacked}}$ is an isotropic, $c_x$-sub-Gaussian random vector. Therefore, by applying \citep[][Theorem~6.3.2]{vershynin2018high} we know that the sub-Gaussian norm~\citep[][Definition~2.5.3]{vershynin2018high} of $\lv X\rv = \lv \Sigma^{1/2}_{\text{stacked}} u_{\text{stacked}} \rv$ is
\begin{align*}
    \left\lv \lv \Sigma^{1/2}_{\text{stacked}} u_{\text{stacked}}  \rv - c_1\sqrt{n\Tr(\Sigma)}\right\rv_{\psi_2} = \left\lv \lv X \rv - c_1\sqrt{ns_0}\right\rv_{\psi_2} \le c_2 \sqrt{\lambda_1}.
\end{align*}
Therefore, by Hoeffding's bound~\citep[][Proposition~2.5.2]{vershynin2018high} we get that
\begin{align*}
    \Pr\left[\lv X\rv - c_1\sqrt{ns_0} \ge \eta \right] \le 2\exp(-c_3\eta^2 /\lambda_1).
\end{align*}
Setting $\eta^2 = n s_0/\lambda_1 = n r_0$ and noting that $n \ge \log(1/\delta) \ge \log(1/\delta)/r_0$ completes the proof. 
\end{proof}

Finally, we have a simple lemma that bounds the Frobenius norm of the responses $Y$.
\begin{lemma}
\label{l:Y_bound}
There is a constant $c > 0$ such that for any $\delta \in (0,1)$, if
$n \geq c  \log(1/\delta)$, then with probability at least $1 - \delta$, $\lv Y \rv \leq c(\lv X\rv_{op}\lv \Theta^\star\rv + \sqrt{qn})$.
\end{lemma}
\begin{proof} Note that $Y = X\Theta^\star + \Omega$, and therefore
\begin{align}\label{e:y_norm_upper_bound_triangle}
    \lv Y \rv &\le \lv X\Theta^\star \rv + \lv \Omega\rv \le \lv X\rv_{op}\lv \Theta^\star \rv +\lv \Omega\rv,
\end{align}
where the last inequality follows since for any matrices $\lv AB\rv \le \lv A\rv_{op}\lv B\rv$. 
Now each entry in $\Omega \in \mathbb{R}^{n \times q}$ is a zero-mean and $c_y$-sub-Gaussian. Therefore, by Bernstein's bound~\citep[][Theorem~2.8.1]{vershynin2018high},
\begin{align*}
    \Pr\left[\lv \Omega \rv^2 - \E\left[\lv \Omega\rv^2\right] \ge qn \right]\le 2\exp(-c_1 qn).
\end{align*}
Now $\E\left[\lv \Omega \rv^2\right] = n \E\left[\lv y-x\Theta^\star\rv^2\right]\le c_2qn$, by Assumption~\ref{assumption:noise_subgaussian}, and $2\exp(-qn)\le \delta$ since $n \ge c\log(1/\delta) \ge c \log(1/\delta)/q$. Thus, with probability at least $1-\delta$
\begin{align*}
    \lv \Omega \rv^2 \le c_3 qn.
\end{align*}
Combining this with Eq.~\eqref{e:y_norm_upper_bound_triangle} completes the proof.
\end{proof}

With all of the pieces in place we are now ready to prove the theorem. 
\begin{proof}[Proof of Theorem~\ref{t:risk_ab_initialization}]
Define a ``good event'' $\cE$ as the intersection of the following events:
\begin{enumerate}[label={\textbullet\hspace{10pt}$\mathcal{E}_{\arabic*}$,},
                  leftmargin=0.6in]
    \item the excess risk bound stated in Theorem~\ref{t:generalization.deep} holds.
    \item the bounds stated in Proposition~\ref{p:optimization_ab_initialization} hold.
    \item $\lv X \rv_{op} \leq c \sqrt{\lambda_1 n}$.
    \item $\lv X \rv \le c \sqrt{s_0 n}$.
    \item    \label{i:sigma_min} $\sigma_{min}(X) \geq \frac{\sqrt{s_k}}{c}$.
    \item $\lv Y\rv \le c \left(\lv X\rv_{op}\lv \Theta^\star\rv + \sqrt{q n}\right)$.
\end{enumerate}
Now, Theorem~\ref{t:generalization.deep} and Proposition~\ref{p:optimization_ab_initialization} each hold with probability at least $1-c\delta$. 
Lemma~\ref{l:Xop_bound} implies that
the event $\cE_3$ holds with probability at least $1-\delta$. By Lemma~\ref{l:X_fro_bound}, the event $\cE_4$ holds with probability at least $1-\delta$.
For $\cE_5$,
notice that
\begin{align*}
    \sigma_{\min}(X) = \sqrt{\sigma_{\min}(XX^\top)} = \sqrt{\sigma_{\min}(X_{:k}X_{:k}^\top + X_{k:}X_{k:}^\top) }
    \ge \sqrt{\sigma_{\min}( X_{k:}X_{k:}^\top) } = \sigma_{\min}(X_{k:}),
\end{align*}
where $X_{:k}$ are the first $k$ columns of $X$ and $X_{k:}$ are the last $d-k$ columns of $X$. Since $n\ge c\log(1/\delta)$, by \citep[][Lemma~9]{bartlett2020benign} we know that with probability at least $1-\delta$
\begin{align*}
    \sigma_{\min}(X)\ge \sigma_{\min}(X_{k:}) &\ge \sqrt{\frac{s_k}{c_1} \left(1-\frac{c_1 n}{r_k}\right)} \ge \frac{\sqrt{s_k}}{c} \qquad \mbox{(since $r_k \ge bn$ by Definition~\ref{def:k_star})}.
\end{align*}
 Finally, by Lemma~\ref{l:Y_bound} event $\cE_6$ holds with probability at least $1-\delta$. Therefore, by a union bound the good event $\cE$ holds with probability at least $1-c'\delta$.
Let us assume that this event occurs going forward in the proof.

Proposition~\ref{p:optimization_ab_initialization} guarantees that the training process is $c_2(\alpha + 1/L)\beta$-perpetually bounded and the loss converges to zero. Therefore, by applying Theorem~\ref{t:generalization.deep}, the risk is bounded by
\begin{align*}
\mathsf{Risk}(\Theta) & \le \mathsf{Bias(\Theta_{\ell_2})}    +\mathsf{Variance(\Theta_{\ell_2})}+ \mathsf{\Xi},
\end{align*}
where
\begin{align*}
    \mathsf{Bias(\Theta_{\ell_2})}  &\le \frac{cs_k}{n} \lv \Theta^\star \rv^2, \\
    \mathsf{Variance(\Theta_{\ell_2})}  &\le  c q \log(q/\delta)\left(\frac{k}{n} + \frac{n}{R_{k}} \right), \\
   \mathsf{\Xi} 
    & \leq \frac{cs_k}{n}\left[ 
     \lv \Theta^{(0)}_{X^{\perp}}\rv
     +L \alpha (\alpha+1/L)^2 \beta \sqrt{\lambda_1 n}\int_{t=0}^\infty
 \sqrt{\cL(\Theta^{(t)})}\;\mathrm{d}t\right]^2. \numberthis \label{e:xi_bound_theorem_opt}
\end{align*}
In the rest of the proof we shall bound the term $\mathsf{\Xi}$. 

For this, we would like to apply Proposition~\ref{p:optimization_ab_initialization}, 
which we can, since $\alpha \leq 1$,
\begin{align*}
\beta & \geq 
   c_3
    \max\left\{ 1,
      \frac{ 
   \lambda_1^{1/4} \sqrt{L n} \left( \sqrt{\lv \Theta^\star \rv} \lambda^{1/4} + q^{1/4}
      \right)
       }{
       \sqrt{s_k}
       }
      \right\} \\
   &\geq
 c_4  
   \max \left\{ 1,
   \sqrt{\frac{L \lv X\rv_{op}\lv Y\rv}
          { \sigma^2_{\min}(X)}}
          \right\} && \mbox{(by events~$\cE_3$, $\cE_5$ and $\cE_6$)},
\end{align*}
and
\begin{align*}
m & \geq 
c_5\max\left\{  d+q+\log(1/\delta), \frac{ L^2\alpha^2 \lambda_1 s_0 n^2 q\log(n/\delta)}
                     {\beta^2 
                     s_k^2
                     } \right\} \\
 & \geq
c_6\max\left\{  d+q+\log(1/\delta), \frac{L^2 \alpha^2 \lv X\rv_{op}^2 \lv X\rv^2 q\log(n/\delta)}
                     {\beta^2 \sigma^4_{\min}(X)} \right\}
&& \mbox{(by events $\cE_3$ and $\cE_4$)}. 
\end{align*}
Thus, by Proposition~\ref{p:optimization_ab_initialization} we know that for all $t > 0$,
\begin{align*}
     \cL(\Theta^{(t)})
    &  <
    c_7\left(\lv Y\rv^2 + \frac{\alpha^2 \beta^2 q \lv X\rv^2\log(n/\delta)}{m} \right)
    \exp\left(- \frac{\beta^2 \sigma_{\min}^2(X) }{4e}\cdot t \right) \\
    &  <
    c_8 \left(\lambda_1 n \lv \Theta^\star\rv^2 + q n + \frac{\alpha^2 \beta^2 q s_0 n\log(n/\delta)}{m} \right)
    \exp\left(- c_9\beta^2  s_k \cdot t \right)&& \mbox{(by events $\cE_3$-$\cE_6$)}. 
\end{align*}
Integrating the RHS above we get that
\begin{align*}
 \int_{t=0}^\infty 
 \sqrt{\cL(\Theta^{(t)})}\;\mathrm{d}t  &\le c_{10} 
   \frac{\sqrt{(\lambda_1 \lv \Theta^\star\rv^2 + q) n}  
            + \alpha \beta \sqrt{\frac{ s_0 q n \log(n/\delta)}{m}}}
        {\beta^2 s_k}. \numberthis \label{e:loss_bound_1_infinity}
\end{align*}

Proposition~\ref{p:optimization_ab_initialization} also guarantees that the initialization is $\delta$-good. That is, $\lv W_1^{(0)}\rv_{op} \le c_{11}\alpha $ and $\lv W_L^{(0)}\rv\le c_{11}\beta$. So,
\begin{align*}
\lv \Theta_{X^{\perp}}^{(0)} \rv
 = \lv (I - P_X) \Theta^{(0)} \rv
 &\leq \lv (I - P_X) \rv_{op} \lv \Theta^{(0)} \rv \\
 & \le \lv \Theta^{(0)} \rv  \\
 & = \lv W_L^{(0)}W_{L-1}^{(0)}\cdots W_1^{(0)} \rv  \\
 & = \lv W_L^{(0)}W_1^{(0)} \rv
 % &&
 \hspace{0.7in}
 \mbox{(since $W_2^{(0)}=\ldots=W_{L-1}^{(0)}=I$)} \\
 & \le \min\left\{\lv W_1^{(0)} \rv_{op} \lv W_L^{(0)}\rv,\lv W_1^{(0)} \rv \lv W_L^{(0)}\rv_{op}\right\} \\
 & \overset{(i)}{\le}  c_{12}\sqrt{\min\left\{q,d\right\}}\alpha \beta \\
 & \le c_{12} \sqrt{q}\alpha \beta, \numberthis \label{e:theta_0_bound}
\end{align*}
where $(i)$ follows since the initialization was good, and the ranks of $W_{1}^{(0)}$ and $W_{L}^{(0)}$ are bounded by $d$ and $q$ respectively.

Plugging the bounds obtained in Eqs.~\eqref{e:loss_bound_1_infinity} and \eqref{e:theta_0_bound} into Eq.~\eqref{e:xi_bound_theorem_opt} we have that
\begin{align*}
    \mathsf{\Xi}
    & \le \frac{c_{13} s_k}{n}\left[\sqrt{q}\alpha \beta+  
     L \alpha (\alpha + 1/L)^2 \beta^2\sqrt{\lambda_1 n}
    \left(
          \frac{\sqrt{(\lambda_1 \lv \Theta^\star\rv^2 + q) n}  
            + \alpha \beta \sqrt{\frac{ s_0 qn  \log(n/\delta)}{m}}}
        {\beta^2 s_k}
         \right) \right]^2 \\ 
         & \le \frac{c_{13} s_k}{n}\left[\sqrt{q}\alpha \beta+  
          \frac{L \alpha (\alpha + 1/L)^2 \sqrt{\lambda_1 n}\left(\sqrt{(\lambda_1 \lv \Theta^\star\rv^2 + q) n}  
            + \alpha \beta \sqrt{\frac{ s_0 qn  \log(n/\delta)}{m}}\right)}
        { s_k} \right]^2 \\ 
    & \leq \frac{c_{14} \alpha^2 s_k}{n}\left[q  \beta^2+  
          \frac{
           L^2  (\alpha + 1/L)^4 \lambda_1 n
          \left((\lambda_1 \lv \Theta^\star\rv^2 + q) n  
                    +  \frac{ \alpha^2 \beta^2  s_0 q n\log(n/\delta)}{m}\right)}
               {s_k^2}
          \right] \\
          & \leq \frac{c_{14} \alpha^2 s_k}{n}\left[q  \beta^2+  
          \frac{L^2(\alpha+1/L)^4\lambda_1 n^2}{s_k^2}\left(\lambda_1 \lv \Theta^\star \rv^2 +q +\frac{\alpha^2 \beta^2 s_0 q \log(n/\delta)}{m}\right)
          \right] .
    \end{align*}
This completes our proof.
\end{proof}

\section{Proof of Corollary~\ref{c:k_eps_p}}
\label{s:spike_proof}

When $\Sigma$ is an instance of the
$(k,\epsilon)$-spike model we find that
\begin{align}
\label{e:spike_params}
r_0 = s_0/\lambda_1 = k + \epsilon (d - k),\quad  s_k = \epsilon (d-k) \quad \text{and} \quad  R_k = (d - k).
\end{align}
First, for a large enough $c_1$, we set 
\begin{align*}
\beta & =       c_1 \max \left\{1,
      \frac{ \lambda_1^{1/4} \sqrt{Ln}\left( \sqrt{\lv \Theta^\star \rv\lambda_1^{1/4} + q^{1/4}} \right)}{\sqrt{s_k}}
        \right\} = c_1 \max \left\{1,
      \frac{\sqrt{Ln}\left( \sqrt{\lv \Theta^\star \rv + q^{1/4}} \right)}{\sqrt{\epsilon(d-k)}}
        \right\}.
\end{align*} Given this choice of $\beta$, for any $q,n,k,d,L,$ if $\alpha >0$ is chosen to be small enough then,
\begin{align*}
     m\ge c_1(d+q+\log(1/\delta)) = c_1 \max\left\{ d+q+\log(1/\delta),\frac{L^2 \alpha^2 \lambda_1 s_0 n^2 q\log(n/\delta)}
                     {\beta^2 
                     s_k^2}  \right\}.
\end{align*}
Also by the assumption on the number of samples,
\begin{align*}
    n \ge c_2\max\left\{k+\epsilon d, \log(1/\delta)\right\} \ge c_2 \max\left\{r_0, k , \log(1/\delta) \right\}.
\end{align*}
We are now in position to invoke Theorem~\ref{t:risk_ab_initialization}. By this theorem we get that,
\begin{align*}
    \mathsf{Risk}(\Theta) &\le \frac{c_3 \epsilon (d-k)}{n} \lv \Theta^\star\rv^2 + c_3q\log(q/\delta)\left(\frac{k}{n}+\frac{n}{d-k}\right) + \mathsf{\Xi}\\
    &\le \frac{c \epsilon d}{n} \lv \Theta^\star\rv^2 + cq\log(q/\delta)\left(\frac{k}{n}+\frac{n}{d}\right) + \mathsf{\Xi} && \mbox{(since $d \ge c_4 k$)}.
\end{align*}
Recall from above that the upper bound on $\mathsf{\Xi}$ scales with $\alpha^2$. Thus, for small enough $\alpha$ it is a lower order term.

\section{Additional Simulations and Details}\label{s:simulations}

\begin{figure}[H]
    \centering
    \includegraphics[width=\textwidth]{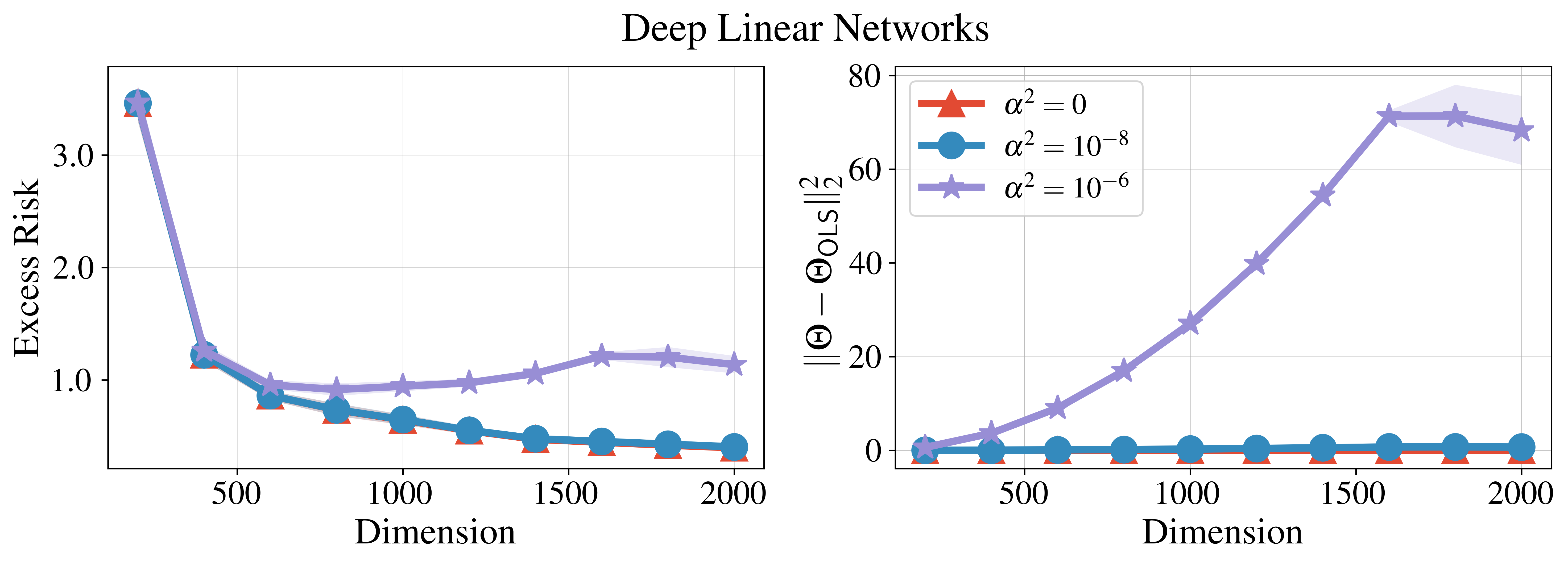}
    \caption{Excess risk and distance from the minimum $\ell_2$-norm interpolator of three-layer linear networks trained by gradient descent on data generated by an underlying linear model as the input dimension varies. The model is trained on $n=100$ points drawn from the generative model $y = x\Theta^\star + \omega$, where $x \sim \mathsf{N}(0,\Sigma)$ and $\omega \sim \mathsf{N}(0,1)$. The excess risk is defined as $\E_{x}\left[\lv x\Theta-x\Theta^{\star}\rv^2\right]$. In line with our theory, we find that when the initialization scale is small, final solution is close to the minimum $\ell_2$-norm interpolator and the resulting excess risk is small. }
    \label{fig:linear_models_risk_versus_dimension}
\end{figure}
Inspired by our theory, we ran simulations to study the excess risk of several linear networks and ReLU networks as a function of both the initialization scale and dimension.\footnote{Code at \url{https://github.com/niladri-chatterji/Benign-Deep-Linear}} In line with our theoretical upper bounds, we find that for deep linear networks as the initialization scale of either the first layer $(\alpha)$ or the last layer $(\beta)$ is large, the excess risk of the model is larger (see Figure~\ref{fig:linear_models_risk_versus_init_scale}). In deep ReLU networks (see Figure~\ref{fig:relu_models_risk_versus_alpha}), we find an asymmetry in the roles of $\alpha$ and $\beta$. The excess risk increases when we increase $\alpha$, but is largely unaffected by the scale of the initialization of the final layer $\beta$.

In all of our figures we report the average over $20$ runs. We also report the $95\%$ confidence interval assuming that the statistic of interest follows a Gaussian distribution.

\paragraph{Setup for deep linear models.} For Figures~\ref{fig:linear_models_risk_versus_init_scale} and \ref{fig:linear_models_risk_versus_dimension} the generative model for the underlying data was $y = x \Theta^\star + \omega$, where 
\begin{enumerate}
    \item $\Theta^\star \in \mathbb{R}^{d \times 3}$ is drawn uniformly over the set of matrices with unit Frobenius norm. The output dimension $q=3$;
    \item the covariates $x \sim \mathsf{N}(0,\Sigma)$, where the eigenvalues of $\Sigma$ are as follows: $\lambda_1 = \ldots = \lambda_{10} = 1$ and $\lambda_{11}=\ldots=\lambda_{d} = 0.01$;
    \item the noise $\omega$ is drawn independently from $\mathsf{N}(0,1)$. 
\end{enumerate}
For these figures the number of samples $n=100$ across all experiments. All of the models are trained on the squared loss with full-batch gradient descent with step-size $10^{-4}$, until the training loss is smaller than $10^{-7}$.

We train models that have $2$ hidden layers $(L=3)$. The width of the middle layers $m$ is set to be $10(d+q)$, where $d$ is the input dimension and $q$ is the output dimension.

For the top half of Figure~\ref{fig:linear_models_risk_versus_init_scale} and Figure~\ref{fig:linear_models_risk_versus_dimension} when we vary the initialization scale of the first layer $\alpha$, we initialize all of the middle layers to the identity, and initialize entries of the last layer with i.i.d. draws from $\mathsf{N}(0,1)$. 

For the bottom half of Figure~\ref{fig:linear_models_risk_versus_init_scale} when we vary the initialization scale of the last layer $\alpha$, we initialize all of the middle layers to the identity, and initialize entries of the first layer with i.i.d. draws from $\mathsf{N}(0,1)$. 

\paragraph{Setup for deep ReLU models.} For Figure~\ref{fig:relu_models_risk_versus_alpha} the generative model for the underlying data was $y =f^\star(x) + \omega$, where 
\begin{enumerate}
    \item $f^\star(x)$ is a two-layer feedforward ReLU network with width $10$ and output dimension $3$ which was randomly initialized according to LeCun initialization;
    \item the covariates $x \sim \mathsf{N}(0,I_{10\times 10})$;
    \item the noise $\omega$ is drawn independently from $\mathsf{N}(0,1)$. 
\end{enumerate}
The networks are trained on $n=500$ samples. Again, all of the models are trained on the squared loss with full-batch gradient descent with step-size $10^{-4}$, until the training loss is smaller than $10^{-7}$.

We train models that have $L=3$ layers. The width of the middle layers $(m)$ is set to be $50$.

For left half of Figure~\ref{fig:relu_models_risk_versus_alpha} when we vary the initialization scale of the first layer $\alpha$, we initialize all of the middle layers to the identity, and initialize entries of the last layer with i.i.d. draws from $\mathsf{N}(0,1)$. 

For the right half of Figure~\ref{fig:relu_models_risk_versus_alpha} when we vary the initialization scale of the last layer $\alpha$, we initialize all of the middle layers to the identity, and initialize entries of the first layer with i.i.d. draws from $\mathsf{N}(0,1)$. 

\section{Discussion}\label{s:discussion}

We have provided upper bounds on the excess risk for deep linear
networks that interpolate the data with respect to the quadratic
loss, and presented simulation studies that verify that the some
aspects of our bounds reflect typical behavior.

As mentioned in the introduction, our analysis describes a variety
of conditions under which the generalization behavior of interpolating
deep linear networks is similar, or the same, as the behavior of the
minimum $\ell_2$-norm interpolant with the standard parameterization.  Among other things,
this motivates study of loss functions other than the quadratic loss
used in this work.  The softmax loss would be a natural choice.

Looking at our proofs, it appears that the only way that a deep linear
parameterization can promote benign overfitting is for the function
computed by the network at initialization
to approximate the regression function.  
(Formalizing this with a lower bound, possibly in the case of
random initialization, or with an arbitrary initialization and a randomly chosen
regression function $\Theta^\star$, is a potential topic for
further research.)
The benefits of a good approximation to the regression function at
initialization has been explored in the case of two-layer linear networks~\citep{chatterji2022interplay}.  Extending this analysis to
deep networks is a potential subject for further study.

We focused on a particular random initialization scheme in this paper, it is possible to study other initialization schemes as well. 
For example, 
%LMC \citet{DBLP:conf/iclr/ZouLG20} considered
% the ``modified identity'' initialization, in which
% \begin{itemize}
%     \item $m = d + q$, 
%     \item $W_1$ copies the inputs onto the first $d$ components
%     of the first hidden layer of nodes,
%     \item $W_2 = \ldots = W_{L-1} = I$, and
%     \item $W_L$ copies the last $q$ components of the layer hidden layer
%     of nodes onto the output.
% \end{itemize}
% Since the function computed by this network is identically zero, 
% a slight modification of the analysis in Appendix~\ref{a:optimization_ab_initialization}
% (which in turn depends heavily on \citep{DBLP:conf/iclr/ZouLG20}),
% together with Theorem~\ref{t:generalization.deep},
% yields the same bound as the minimum $\ell_2$-norm interpolant for the network learnt by gradient flow with the modified
% identity initialization.  (Indeed its proof shows that
% the two procedures produce exactly the same hypothesis.)
% We also 
we
believe that, if the width $m$ of the network
is somewhat larger, a similar analysis should go
through without our simplifying assumption that $W_2, \ldots, W_{L-1}$ 
are initialized exactly to the identity, and 
instead are initialized randomly.

Recently, \citet{Mal22} established conditions under
which interpolation with the minimum $\ell_2$-norm intepolator is ``tempered'', 
achieving risk within a constant factor of the Bayes risk. 
Here we show that the risk of interpolating deep
linear networks is (nearly) equal to the risk of the minimum $\ell_2$-norm interpolator, this implies 
that when the minimum $\ell_2$-norm interpolator is tempered, so is the output of
the deep linear model.  We hope that our techniques
lay the groundwork for other results about tempered
overfitting.

While here we analyzed the network obtained by the continuous-time gradient flow it is straightforward to use our techniques to obtain similar results for gradient descent with small enough step-size at the expense of a more involved analysis.

As mentioned after the statement of Theorem~\ref{t:risk_ab_initialization},
its bounds could potentially be improved.  (We have not attempted to prove any lower bounds in this work.)

In Figure~\ref{fig:linear_models_risk_versus_init_scale} we found that when the scale of the random initialization of either the first $(\alpha)$ or the last layer $(\beta)$ goes to zero the trained model approach the minimum $\ell_2$-norm interpolator. Understanding why and when this happens is an avenue for future research.

Finally, examining the extent to which the effects described here carry over when nonlinearities are present is a natural next step.  

\subsection*{Acknowledgements}
We thank anonymous reviewers for their careful reading of an earlier version
of this
paper and their valuable feedback.

%\bibliographystyle{plain}
%\bibliography{ref}

\appendix

\section{Additional Related Work}\label{s:additional_related_work}

In this appendix, we describe a wider variety of related work.

%LMC (paragraph not allowed by JMLR) 

\subsection{Benign Overfitting and Double Descent} 
This subsection includes descriptions of some of the most closely related work
that we know. For a wider sample, we point the interested reader to a couple of surveys~\citep{bartlett2021deep,belkin2021fit}.

Papers have studied the excess risk of the minimum $\ell_2$-norm interpolant~\citep{bartlett2020benign,hastie2019surprises,muthukumar2020harmless,bartlett2020failures}, which is obtained as a result of minimizing the squared loss using gradient descent with no explicit regularization. While these previous papers directly analyzed the closed form expression of the minimum $\ell_2$-norm interpolant, followup work~\citep{negrea2020defense,koehler2021uniform,chinot2020robustness,chinot2020robustness_min} employed tools from uniform convergence to analyze its excess risk. Prior work~\citep{kobak2020optimal,NEURIPS2020_72e6d323,tsigler2020benign} analyzed ridge regression with small or even negative regularization, and identified settings where using zero or even negative regularization can be optimal.

Techniques have also been developed to upper bound the excess risk of the sparsity-inducing minimum $\ell_1$-norm interpolant~\citep{koehler2021uniform,li2021minimum,wang2022tight,donhauser22fast}. Furthermore, lower bounds on the excess risk that show that sparsity can be incompatible with benign overfitting, and that the excess risk of sparse interpolators maybe exponentially larger than that of dense interpolators have also been derived~\citep{chatterji2022foolish}.

Kernel ridgeless regression has been actively studied~\citep{liang2020just,mei2019generalization}. Careful theoretical analysis and simulations have revealed that kernel ``ridgeless'' regression can lead to multiple descent curves~\citep{liang2020MultipleDescent}. A handful of papers have analyzed the risk of random features models~\citep{mei2019generalization,li2021towards}.

Several papers have also studied benign overfitting in linear classification of the canonical maximum $\ell_2$-margin classifier~\citep{montanari2019generalization,deng2019model,chatterji2020finite,hsu2020proliferation,muthukumar2020classification,wang2021benign,cao2021risk}, the maximum $\ell_1$-margin classifier~\citep{liang2020precise}, and classifiers obtained by minimizing polynomially-tailed classification losses~\citep{wang2021importance}. Results have also been obtained on data that is linearly separable with two-layer leaky ReLU networks~\citep{frei2022benign}, and with two-layer convolutional networks with smooth nonlinearities~\citep{cao2022benign}.

Furthermore, this phenomenon has been studied in nearest neighbor models~\citep{belkin2018overfitting}, latent factor models~\citep{bunea2020interpolation} and the Nadaraya-Watson estimator~\citep{belkin2019does} 
with a singular kernel.

% Finally, when studying benign overfitting,
\citet{shamir2022implicit} highlighted the importance of the choice of the loss function, since an interpolator that benignly overfits with respect one loss function may not with respect to another one.

\subsection{Implicit Bias}
In the closely related problem of matrix factorization, several papers~\citep{gunasekar2017implicit,arora2019implicit} established conditions under which the solution of gradient flow converges to the minimum nuclear norm solution. This rank minimization behavior of gradient flow was also shown to approximately hold for ReLU networks~\citep{timor2022implicit}. 

Papers have also studied the implicit bias of gradient descent for \emph{diagonal} linear networks and found that it depends on the scale of the initialization and step-size~\citep{woodworth2020kernel,DBLP:conf/iclr/YunKM21,nacson2022implicit}. They found that depending on the scale of the initialization and step-size the network converges to the minimum $\ell_2$-norm solution (kernel regime), or to the minimum $\ell_1$-norm solution (rich regime) or to a solution that interpolates between these two norms. \citet{gunasekar2018implicit,DBLP:conf/colt/JagadeesanRG22} studied linear convolutional networks and found that under this parameterization of a linear model, gradient descent implicitly minimizes norms of the Fourier transform of the predictor. Techniques have also been developed to study the implicit bias of mirror descent~\citep{gunasekar2018characterizing,li2022implicit}.

For linear classifiers, minimizing 
exponentially-tailed losses including the logistic loss
leads to the maximum $\ell_2$-margin classifier~\citep{soudry2018implicit,ji2018risk,nacson2018stochastic}. \citet{ji2018gradient} found that this is also the case when the linear classifier is parameterized as a deep linear classifier.

As a counterpoint to this line of research, 
\citet{arora2019implicit} and
\citet{razin2020implicit} raised the possibility that the implicit bias of deep networks may be unexplainable by 
a simple function such as a norm. 
\citet{litowards} showed that
under certain conditions gradient flow with infinitesimal initialization is equivalent to a simple heuristic rank minimization algorithm.
\citet{vardi2021implicit} showed that it might be impossible altogether to capture the implicit bias of even two-layer ReLU networks using any functional form.

\subsection{Optimization of Deep Linear Networks} 
Several papers have studied deep linear networks as a means to understand the benefit of overparameterization while optimizing nonlinear networks. \citet{DBLP:conf/icml/AroraCH18} argued that depth promotes a form of implicit acceleration when performing gradient descent on deep linear networks. While other papers~\citep{DBLP:conf/icml/DuH19} showed that for wide enough networks that are randomly initialized by Gaussians the loss converges at a linear rate. Other papers have analyzed the convergence rate under other initialization schemes such as orthogonal initialization~\citep{DBLP:conf/iclr/HuXP20} and near-identity initialization~\citep{DBLP:journals/neco/BartlettHL19,DBLP:conf/iclr/ZouLG20}. Rates of convergence for accelerated methods such as Polyak's heavy ball method have also been established~\citep{wang2021modular}.

\citet{DBLP:journals/corr/SaxeMG13} analyzed the effect of initialization and step-size in training deep linear networks. 
\citet{DBLP:conf/nips/Kawaguchi16} identified a number of properties of the
loss landscape of deep linear networks, including the absence of suboptimal
local minima;
\citet{ach21arxiv} 
characterized global minima, strict saddles, and non-strict saddles for 
these landscapes.

\section{Proof of Proposition~\protect\ref{p:optimization_ab_initialization}}
\label{a:optimization_ab_initialization}

In this appendix, we prove
Proposition~\ref{p:optimization_ab_initialization}.  Its proof uses some technical lemmas, which
we derive first.

It is useful to recall the definition of a $\delta$-good initialization from above. 
\initdef*
The next two lemmas shall be useful in showing that our initialization scheme leads to a $\delta$-good initialization. First, to bound the singular values of the weight matrices at initialization we apply the following result
from \citet{vershynin2010introduction}.
\begin{lemma}
\label{l:singular.deep_random}There exists a constant $c$ such that given any $\delta \in (0,1)$ if $m \ge c(d+q+\log(1/\delta))$, then 
with probability at least $1 - \delta/2$
\begin{align*}
    \alpha/2 < \sigma_{\min}(W_{1}^{(0)}) &\le  \sigma_{\max}(W_{1}^{(0)}) < 2 \alpha \qquad \text{and}\\
  \beta/2  < \sigma_{\min}(W_{L}^{(0)}) &\le \sigma_{\max}(W_{L}^{(0)}) < 2 \beta.
\end{align*}
\end{lemma}
\begin{proof} By \citep[][Corollary~5.35]{vershynin2010introduction} we have that 
with probability at least $1 - 4e^{-\eta^2/2}$
\begin{align*}
     \frac{\alpha}{\sqrt{m}}\left(\sqrt{m}- \sqrt{d}-\eta\right) \le \sigma_{\min}(W_{1}^{(0)}) &\le \sigma_{\max}(W_{1}^{(0)}) \le \frac{\alpha}{\sqrt{m}}\left(\sqrt{m}+ \sqrt{d}+\eta\right) \quad \text{and}\\
     \frac{\beta}{\sqrt{m}}\left(\sqrt{m}- \sqrt{q}-\eta\right) \le \sigma_{\min}(W_{L}^{(0)}) &\le \sigma_{\max}(W_{L}^{(0)}) \le \frac{\beta}{\sqrt{m}}\left(\sqrt{m}+ \sqrt{q}+\eta\right).
\end{align*}
So since $m\ge c(d+q+\log(1/\delta))$, where $c$ is a large enough constant, by picking $\eta = \sqrt{m}/8$ we get that with probability at least $1-\delta/2$
\begin{align*}
     \frac{\alpha}{2} < \sigma_{\min}(W_{1}^{(0)}) &\le \sigma_{\max}(W_{1}^{(0)}) < 2\alpha\\
     \frac{\beta}{2} < \sigma_{\min}(W_{L}^{(0)}) &\le \sigma_{\max}(W_{L}^{(0)}) < 2\beta,
\end{align*}
completing the proof.
\end{proof}

The next lemma shows that the loss is controlled at initialization.
\begin{lemma}
\label{l:L0.deep.ab}There is a positive constant $c$ such that, for any $\delta \in (0,1)$, provided that $m \ge c\left(d+q+\log(1/\delta)\right)$, with probability at least $1-\delta/2$
\begin{align*}
    \cL(\Theta^{(0)}) < c\left(\lv Y\rv^2 + \frac{\alpha^2 \beta^2 q \lv X\rv^2\log(n/\delta)}{m} \right).
\end{align*}
\end{lemma}
\begin{proof}
The lemma directly follows by invoking \citep[][Proposition~3.3]{DBLP:conf/iclr/ZouLG20}.   
\end{proof}

The next lemma shows that if the weights remain close to
their initial values then the loss decreases at a certain rate.
\begin{lemma} \label{l:loss_decreases_alpha_deep}
If $\beta \geq 2$ and there was a good initialization (see Definition~\ref{def:good_ab_initialization}), at any time $t > 0$ if, for all $j \in [L]$
\begin{align*}
    \lv W^{(t)}_j -W^{(0)}_j \rv_{op} < \frac{1}{2 L} 
\end{align*}
then
    \begin{align*}
    \frac{\mathrm{d}\cL(\Theta^{(t)})}{\mathrm{d} t} < -\frac{\beta^2 \smin^2(X) 
   }{4 e} \cL(\Theta^{(t)}) .
\end{align*}
\end{lemma}
\begin{proof}
By the chain rule, we have that
\begin{align*}
 \frac{\mathrm{d}\cL(\Theta^{(t)})}{\mathrm{d} t} 
 & =\nabla_{\Theta}\cL(\Theta^{(t)}) \cdot \frac{\mathrm{d}\Theta}{\mathrm{d} t} 
  =-\nabla_{\Theta}\cL(\Theta^{(t)}) \cdot \nabla_{\Theta}\cL(\Theta^{(t)})
  = - \lv \nabla_{\Theta}\cL(\Theta^{(t)})\rv^2.
\end{align*}
Further, observe that
\begin{align*}
  \left\lv \nabla_{\Theta}\cL(\Theta^{(t)})\right\rv^2 
 & = \sum_{j=1}^L \left\lv \nabla_{W_j}\cL(\Theta^{(t)})\right\rv^2 \\
 & \geq \left\lv \nabla_{W_1}\cL(\Theta^{(t)})\right\rv^2 \\
 & = \left\lv (W_L^{(t)} \cdots W_{2}^{(t)})^{\top}(X\Theta - Y)^{\top} X \right\rv^2 .
 \end{align*}
 Continuing by applying \citep[][Lemma~B.3]{DBLP:conf/iclr/ZouLG20} to the RHS of the inequality above we get that,
 \begin{align*}
 \left\lv \nabla_{\Theta}\cL(\Theta^{(t)})\right\rv^2 
 & \ge \smin^2(X) \left\lv X\Theta - Y \right\rv^2\left( \prod_{k \neq 1} \smin^2(W_k^{(t)}) \right)  \\
 & = \smin^2(X) \cL(\Theta^{(t)})  \left( \prod_{k \neq 1} \smin^2(W_k^{(t)}) \right) \\
 & \overset{(i)}{\geq} \smin^2(X) \cL(\Theta^{(t)}) \left( \prod_{k \neq 1}
   \left(
 \smin(W_k^{(0)}) - \lv W_k^{(t)} -W_k^{(0)} \rv_{op} 
   \right)^2
   \right) \\
 & \overset{(ii)}{>}  \smin^2(X)  \cL(\Theta^{(t)}) 
    \left(\frac{\beta}{2} - \frac{1}{2L}\right)^2 
     \left(1 - \frac{1}{2L} \right)^{2 (L-2)} \\
 & = \frac{\beta^2}{4} \smin^2(X) \cL(\Theta^{(t)}) 
    \left(1 - \frac{1}{\beta L}\right)^2 
     \left(1 - \frac{1}{2L} \right)^{2 (L-2)} \\
 & \geq \frac{\beta^2}{4} \smin^2(X) \cL(\Theta^{(t)}) 
     \left(1 - \frac{1}{2L} \right)^{2 L-2} \\
 & \geq  \frac{\beta^2 \smin^2(X) 
   }{4 e}  \cL(\Theta^{(t)}),
\end{align*}
where  $(i)$ follows since on a good initialization for all $j\in [L]$, $\sigma_{\min}(W_j^{(0)})\ge \min\{\beta/2,1\}\ge 1$ and by assumption $\lv W_{j}^{(t)}-W_{j}^{(0)}\rv_{op}< 1/(2L)$. Inequality~$(ii)$ follows since $\beta \geq 2$, and there was a good initialization (which implies that the event in Lemma~\ref{l:singular.deep_random} occurs). 
\end{proof}

The next lemma shows that if the weight matrices remain close throughout the path of gradient flow then the loss decreases.

\begin{lemma}
\label{l:loss_bound_given_norm_bound_alpha_deep}If $\beta \ge 2$ and there was a good initialization (see Definition~\ref{def:good_ab_initialization}), given any $t>0$ if for all $0\le s< t$ and 
for all $j \in [L]$
\begin{align*}
    \lv W^{(s)}_j -W^{(0)}_j \rv_{op} < \frac{1}{2 L} 
\end{align*}
then, 
\begin{align*}
   \cL(\Theta^{(t)})< \cL(\Theta^{(0)})\exp\left(- \frac{\beta^2 \sigma_{\min}^2(X) }{4e}\cdot t \right).
\end{align*}
\end{lemma}
\begin{proof} By assumption $\lv W_{j}^{(s)}-W_{j}^{(0)}\rv_{op}< 1/(2L)$ for all $0\le s<t$. Thus, by invoking Lemma~\ref{l:loss_decreases_alpha_deep} we know that for all $0\le s < t$,
\begin{align*}
    \frac{\mathrm{d}\cL(\theta^{(s)})}{\mathrm{d}s} < -\frac{\beta^2  \sigma^2_{\min}(X)}{4e}\mathcal{L}(\Theta^{(s)}).
\end{align*}
which implies 
\begin{align*}
  \int_{0}^t \frac{1}{\cL(\Theta^{(s)})}\frac{\mathrm{d}\cL(\theta^{(s)})}{\mathrm{d}s} \;\mathrm{d}s < -\int_{0}^t 
  \frac{\beta^2 \smin^2(X) 
   }{4 e}   \;\mathrm{d}s.
\end{align*}
Integrating both sides we get that 
\begin{align*}
    \cL(\Theta^{(t)})< \cL(\Theta^{(0)})\exp\left(- \frac{\beta^2 \sigma_{\min}^2(X) }{4e}\cdot t \right).
\end{align*}
\end{proof}

We also need the lemma that controls the growth of the operator norm of $W_{j}^{(t)}$.
\begin{lemma}\label{l:norm_growth_control_alpha_deep}
There is a positive absolute constant $c$ such that if 
\begin{itemize}
    \item $\beta \geq c \max\left\{1, \sqrt{\frac{L \lv X\rv_{op}\lv Y\rv}
          { \sigma^2_{\min}(X)}}  \right\}$ and
    \item $ m \ge c \frac{L^2 \alpha^2 \lv X\rv_{op}^2 \lv X\rv^2 q\log(n/\delta)}
                     {\beta^2 \sigma^4_{\min}(X)}$,
\end{itemize}
    then on a good initialization, given any $t>0$ if for all $0\le s< t$
\begin{align*}
     \lv W^{(s)}_j -W^{(0)}_j \rv_{op} < \frac{1}{2 L}
\end{align*}
then 
\begin{align*}
    \lv W^{(t)}_j -W^{(0)}_j \rv_{op} < \frac{1}{2 L}.
\end{align*}
\end{lemma}
\begin{proof}
Applying  
\citep[][Lemma~A.1]{DBLP:conf/iclr/ZouLG20}
with $A = B = I$ 
we have, for all $s\ge 0$,
\begin{align*}
\left\lv \nabla_{W_j}\cL(\Theta^{(s)})
\right\rv^2 \le 2e\lv X \rv_{op}^2 \cL(\Theta^{(s)}).
\end{align*}
% We know that b
By the definition of gradient flow,
\begin{align*}
     W_j^{(t)} -W_j^{(0)}& = -\int_{0}^t \nabla \cL_{W_j}(\Theta^{(s)}) \; \mathrm{d}s,
\end{align*}
which implies that,
\begin{align*}
\lv W_j^{(t)} -W_j^{(0)} \rv_{op}
 &\leq \int_{0}^t \lv 
      \nabla \cL_{W_j}(\Theta^{(s)}) \rv_{op}\; \mathrm{d}s \\
      &\leq \int_{0}^t \sqrt{2e}\lv X\rv_{op}\sqrt{\cL(\Theta^{(s)})}\; \mathrm{d}s \\
      &\overset{(i)}{<} \sqrt{2e}\lv X\rv_{op}\int_{0}^t \sqrt{\cL(\Theta^{(0)})} \exp\left(-\frac{\beta^2 \sigma_{\min}^2(X)}{4e}\cdot s\right)\; \mathrm{d}s \\
       &=  \sqrt{2e}\lv X\rv_{op}\sqrt{\cL(\Theta^{(0)})}\int_{0}^t  \exp\left(-\frac{\beta^2  \sigma_{\min}^2(X)}{4e}\cdot s\right)\; \mathrm{d}s \\
      & = \frac{4e\sqrt{2e}\lv X\rv_{op}\sqrt{\cL(\Theta^{(0)})}}{\beta^2  \sigma^2_{\min}(X)} \left[1-\exp\left(-\frac{\beta^2 \sigma_{\min}^2(X) t}{4e}\right)\right] \\
      & < \frac{30\lv X\rv_{op}\sqrt{\cL(\Theta^{(0)})}}{\beta^2 \sigma^2_{\min}(X)} ,
\end{align*}
where $(i)$ follows by applying Lemma~\ref{l:loss_bound_given_norm_bound_alpha_deep}. 

On a good initialization we have that
\begin{align*}
    \cL(\Theta^{(0)}) < c_1\left(\lv Y\rv^2 + \frac{\alpha^2 \beta^2 q \lv X\rv^2\log(n/\delta)}{m} \right).
\end{align*}

Plugging this into the previous inequality we find that
\begin{align*}
   \lv W_j^{(t)} -W_j^{(0)} \rv_{op} 
    & < \frac{30\lv X\rv_{op}}{\beta^2  \sigma^2_{\min}(X)} \sqrt{c_1\left(\lv Y\rv^2 + \frac{\alpha^2 \beta^2 q \lv X\rv^2\log(n/\delta)}{m} \right)}\\
    & \le c_2 \left(\frac{\lv X\rv_{op}\lv Y\rv}{\beta^2  \sigma^2_{\min}(X)} + \frac{\alpha \lv X\rv_{op} \lv X\rv}{\beta  \sigma^2_{\min}(X)} \sqrt{\frac{ q \log(n/\delta)}{m} }\right).
\end{align*}

So for our lemma to be satisfied it suffices if
\begin{align*}
    c_2 \frac{\alpha \lv X\rv_{op}\lv X\rv}{\beta \sigma^2_{\min}(X)}\sqrt{\frac{q\log(n/\delta)}{m}} < \frac{1}{4L} 
    \Leftrightarrow 
    m > c_3 \frac{L^2 \alpha^2 \lv X\rv_{op}^2 \lv X\rv^2 q\log(n/\delta)}
                     {\beta^2 \sigma^4_{\min}(X)}
\end{align*}
and 
\begin{align*}
 c_2 \frac{\lv X\rv_{op}\lv Y\rv}
          {\beta^2 \sigma^2_{\min}(X)} < \frac{1}{4 L}
 \Leftrightarrow
 \beta > c_4 \sqrt{\frac{L \lv X\rv_{op}\lv Y\rv}
          { \sigma^2_{\min}(X)}}
\end{align*}
\end{proof}
Armed with these lemmas, we are now ready to prove the main result
of this appendix. Recall its statement from above.
\optprop*
\begin{proof} We will first prove that Part~$1$ of the proposition holds with probability $1-\delta$ and then prove the other two parts assuming that the initialization was good.

 \paragraph{Proof of Part~$1$.}By invoking Lemmas~\ref{l:singular.deep_random} and~\ref{l:L0.deep.ab} we have that a good initialization holds with probability at least $1-\delta$.  

\paragraph{Proof of Parts~$2$ and $3$.} Assume that the initialization was good. We claim that, for all $t > 0$,
\begin{itemize}
    \item for all $j \in [L]$, $\lv W_j^{(t)} - W_j^{(0)} \rv_{op} < \frac{1}{2L}$ and
    \item $\cL(\Theta^{(t)})
   < \cL(\Theta^{(0)})\exp\left(- \frac{\beta^2 \sigma_{\min}^2(X) }{4e}\cdot t \right)$.
\end{itemize}
Assume for contradiction that this does not hold.  Since
$\lv W_j^{(t)} - W_j^{(0)} \rv_{op} - \frac{1}{2L}$ and
$\cL(\Theta^{(t)})- \cL(\Theta^{(0)})\exp\left(- \frac{\beta^2 \sigma_{\min}^2(X) }{4e}\cdot t \right)$
are continuous functions of $t$, by the Intermediate Value Theorem there is a least
value $t_0 > 0$ such that one of these quantities equals zero.
But Lemmas~\ref{l:loss_bound_given_norm_bound_alpha_deep} and \ref{l:norm_growth_control_alpha_deep} contradict this. This proves the first bound on the loss in Part~$3$. For the second bound on the loss, we note that on a good initialization
\[
\cL(\Theta^{(0)})< c'\left(\lv Y\rv^2 + \frac{\alpha^2 \beta^2 q \lv X\rv^2\log(n/\delta)}{m} \right).
\]

Finally, we show that the training is perpetually $c(\alpha + 1/L) 
\beta$ bounded. 
Recall that this means that,
for all $t\ge 0$, for all $S \subseteq [L]$, $$\prod_{k \in S} \left\lv W_{k}^{(t)}\right\rv_{op}\le c(\alpha + 1/L) 
\beta.$$
This follows since on a good initialization, 
$\lv W_{1}^{(0)}\rv_{op} \leq 2 \alpha$,
$\lv W_{L}^{(0)}\rv_{op}\le 2\beta$ and for all $j \in \{2,\ldots,L-1\}$, $\lv W_{j}^{(0)} \rv_{op} = 1$. Further, by Lemma~\ref{l:norm_growth_control_alpha_deep} we know that for all $t\ge 0$, $\lv W_{j}^{(t)}-W_{j}^{(0)}\rv_{op}\le 1/(2L)$. Putting these two facts together, proves that the process is perpetually $c(\alpha + 1/L) 
\beta$ bounded.

\end{proof}

\printbibliography

\end{document}

%% file: setting.tex
\usepackage{epsf}
\usepackage{fancyhdr}
\usepackage{graphics}
\usepackage{graphicx}
\usepackage{psfrag}
\usepackage[T1]{fontenc}
\usepackage{color}
\usepackage{amsthm}

\usepackage{libertine}
\usepackage{libertinust1math}
\usepackage{amssymb}

\usepackage[T1]{fontenc}
\usepackage{amsfonts}
\usepackage{amsmath}
\usepackage{mathrsfs}
\usepackage{bm}
\usepackage{accents}
\usepackage{listings}
\usepackage{caption}
\usepackage{subcaption}
\usepackage{float}
\usepackage{mleftright}
\usepackage{subcaption}
\usepackage[linesnumbered, ruled, vlined]{algorithm2e}
\usepackage[titletoc,toc,title]{appendix}
\usepackage{enumerate}
\usepackage{verbatim} % verbatim input
\usepackage{csquotes} % context sensitive quotes
\usepackage{upquote} % allow correct quotes in verbatim
\usepackage[bottom]{footmisc} % put footnotes down at the bottom margin
\usepackage{thmtools}
\usepackage[dvipsnames]{xcolor}
\usepackage[english]{babel} % multilinguality
\usepackage[shortlabels]{enumitem}
\usepackage{mathtools}

\usepackage[colorlinks,linkcolor = RawSienna, urlcolor  = RedViolet, citecolor = RoyalBlue, anchorcolor = ForestGreen,bookmarks=False]{hyperref}

\newlength{\widebarargwidth}
\newlength{\widebarargheight}
\newlength{\widebarargdepth}

\makeatletter
\long\def\@makecaption#1#2{
        \vskip 0.8ex
        \setbox\@tempboxa\hbox{\small {\bf #1:} #2}
        \parindent 1.5em  %% How can we use the global value of this???
        \dimen0=\hsize
        \advance\dimen0 by -3em
        \ifdim \wd\@tempboxa >\dimen0
                \hbox to \hsize{
                        \parindent 0em
                        \hfil 
                        \parbox{\dimen0}{\def\baselinestretch{0.96}\small
                                {\bf #1.} #2
                                } 
                        \hfil}
        \else \hbox to \hsize{\hfil \box\@tempboxa \hfil}
        \fi
        }
\makeatother

\makeatletter
\newcounter{manualsubequation}
\renewcommand{\themanualsubequation}{\alph{manualsubequation}}
\newcommand{\startsubequation}{%
  \setcounter{manualsubequation}{0}%
  \refstepcounter{equation}\ltx@label{manualsubeq\theequation}%
  \xdef\labelfor@subeq{manualsubeq\theequation}%
}
\newcommand{\tagsubequation}{%
  \stepcounter{manualsubequation}%
  \tag{\ref{\labelfor@subeq}\themanualsubequation}%
}
\let\subequationlabel\ltx@label
\makeatother

\date{}
\usepackage[style = alphabetic,citestyle=alphabetic,maxbibnames=99,backend=biber,sorting=nyt,natbib=true,backref=true]{biblatex}
\DefineBibliographyStrings{english}{%
  backrefpage = {Cited on page},% originally "cited on page"
  backrefpages = {Cited on pages},% originally "cited on pages"
}
\makeatletter
\renewenvironment{proof}[1][\proofname]{%
  \par\pushQED{\qed}\normalfont%
  \topsep6\p@\@plus6\p@\relax
  \trivlist\item[\hskip\labelsep\bfseries#1\@addpunct{.}]%
  \ignorespaces
}{%
  \popQED\endtrivlist\@endpefalse
}
\makeatother

\newtheorem{theorem}{Theorem}[section]
\newtheorem{lemma}[theorem]{Lemma}

\newtheorem{corollary}[theorem]{Corollary}
\newtheorem{definition}[theorem]{Definition}

\newcommand\numberthis{\addtocounter{equation}{1}\tag{\theequation}}

\newcommand{\Tr}{\mathrm{Tr}}

\newcommand{\cE}{{\cal E}}

\newcommand{\cL}{{\cal L}}

\newcommand{\R}{\mathbb{R}}

\newcommand{\N}{\mathbb{N}}

\newcommand{\E}{\mathbb{E}}

\renewcommand{\Pr}{\mathbb{P}}
\newcommand{\lv}{\lVert}
\newcommand{\rv}{\rVert}

\renewcommand{\epsilon}{\varepsilon}

\DeclareSymbolFont{extraup}{U}{zavm}{m}{n}
\DeclareMathSymbol{\varheart}{\mathalpha}{extraup}{86}
\DeclareMathSymbol{\vardiamond}{\mathalpha}{extraup}{87}

\DeclareMathOperator*{\argmin}{arg\,min}

\renewcommand{\epsilon}{\varepsilon}

\newcommand{\smin}{\sigma_{\min}}

\renewcommand{\epsilon}{\varepsilon}